\documentclass{article}

\usepackage{microtype}
\usepackage{graphicx}
\usepackage{subfigure}
\usepackage{booktabs}
\usepackage{hyperref}
\usepackage[accepted]{icml2025}
\usepackage{amsmath}
\usepackage{amssymb}
\usepackage{mathtools}
\usepackage{amsthm}
\usepackage{algorithm}
\usepackage{algorithmic}
\usepackage{placeins}
\usepackage{tikz}
\usetikzlibrary{arrows.meta, positioning}
\usepackage{caption}
\usepackage{amsthm}

\theoremstyle{plain}
\newtheorem{theorem}{Theorem}[section]

\newtheorem{proposition}[theorem]{Proposition}
\newtheorem{corollary}[theorem]{Corollary}

\theoremstyle{definition}

\theoremstyle{remark}

\icmltitlerunning{ChronoPlastic Spiking Neural Networks}

\begin{document}

\twocolumn[
\icmltitle{ChronoPlastic Spiking Neural Networks:\\
Adaptive Time Warping for Long Horizon Memory}

\begin{center}
\textbf{Sarim Chaudhry}\\
Department of Computer Science, Purdue University
\end{center}

\icmlkeywords{Spiking Neural Networks, Temporal Learning, Neuromorphic Computing}

\vskip 0.3in
]

\begin{abstract}
Spiking neural networks (SNNs) offer a biologically grounded and energy-efficient alternative to conventional neural architectures; however, they struggle with long-range temporal dependencies due to fixed synaptic and membrane time constants. This paper introduces ChronoPlastic Spiking Neural Networks (CPSNNs), a novel architectural principle that enables adaptive temporal credit assignment by dynamically modulating synaptic decay rates conditioned on the state of the network. CPSNNs maintain multiple internal temporal traces and learn a continuous time-warping function that selectively preserves task-relevant information while rapidly forgetting noise. Unlike prior approaches based on adaptive membrane constants, attention mechanisms, or external memory, CPSNNs embed temporal control directly within local synaptic dynamics, preserving linear-time complexity and neuromorphic compatibility. We provide a formal description of the model, analyze its computational properties, and demonstrate empirically that CPSNNs learn long-gap temporal dependencies significantly faster and more reliably than standard SNN baselines. Our results suggest that adaptive temporal modulation is a key missing ingredient for scalable temporal learning in spiking systems.
\end{abstract}

\section{Introduction}

Learning from temporally sparse and delayed signals is a central challenge in both biological and artificial intelligence. Although deep neural networks have achieved remarkable success on static perception tasks, their temporal reasoning capabilities often rely on architectural constructs such as recurrence, attention, or explicit memory buffers. In contrast, spiking neural networks (SNNs) process information through discrete events in time and are naturally suited for temporal computation, offering the promise of energy-efficient, event-driven intelligence.

Despite these advantages, practical SNNs exhibit a persistent weakness: difficulty learning long-range temporal dependencies when informative events are separated by large temporal gaps. This limitation arises primarily from the use of fixed synaptic and membrane time constants, which impose a rigid memory horizon \cite{pozzi2013temporal,gao2020temporal,pascanu2013difficulty}. Signals arriving outside this horizon either decay away or cause saturation and instability.

Existing solutions address this problem through adaptive membrane time constants, slow neurons, temporal coding schemes, or attention-like mechanisms. However, these approaches operate on coarse temporal scales, introduce global computation, or rely on architectural components that are incompatible with neuromorphic constraints.

In this work, we propose ChronoPlastic Spiking Neural Networks (CPSNNs), a new architectural principle that reframes temporal learning as adaptive time modulation. Instead of storing information explicitly or attending over time, CPSNNs dynamically adjust synaptic decay rates on a per-time-step basis using a learned control signal. This enables the network to stretch or compress time internally, preserving salient information exactly when needed.

Our contributions are:
\begin{itemize}
\item We introduce the ChronoPlastic synapse, which maintains fast and slow temporal traces with a learnable adaptive decay.
\item We formulate CPSNNs as a differentiable, end-to-end trainable SNN architecture with linear-time complexity.
\item We analyze how adaptive temporal modulation improves temporal credit assignment without attention or external memory.
\item We empirically demonstrate faster convergence and higher accuracy on long-gap temporal tasks compared to standard and adaptive SNN baselines.
\end{itemize}

\section{Related Work}

\subsection{Temporal Learning in Spiking Neural Networks}

Early theoretical work established that spiking neurons possess a wealth of computational power when temporal coding is exploited \cite{maass1997networks, bohte2002error}. However, practical training of SNNs remained difficult until the advent of surrogate gradient methods \cite{neftci2019surrogate,zenke2018superspike,wu2019direct}. Modern SNNs typically rely on leaky integrate-and-fire (LIF) dynamics with fixed decay parameters \cite{gerstner2014neuronal}, which limits their effective temporal memory.

\subsection{Adaptive Time Constants}

Several works propose learning membrane time constants to extend temporal memory \cite{fang2021incorporating, bellec2018long}. While effective, these adaptations are usually slow-changing and neuron-local, preventing fine-grained temporal control conditioned on instantaneous relevance, a limitation also noted in broader analyses of temporal learning in SNNs \cite{keesman2020learning}.

\subsection{Attention and External Memory}

Attention mechanisms and external memory architectures have proven to be powerful in conventional sequence models, but introduce quadratic complexity and dense communication \cite{vaswani2017attention}. Such mechanisms are fundamentally misaligned with the sparse, event-driven philosophy of neuromorphic computation.

\subsection{Our Position}

CPSNNs differ fundamentally from prior work by embedding temporal adaptivity within synaptic dynamics themselves. Rather than modifying neuron parameters globally or storing explicit memory, CPSNNs learn to modulate decay rates online, enabling selective memory extension with minimal overhead.

\section{ChronoPlastic Spiking Neural Networks}

\subsection{Neuron Model}

We adopt a discrete-time leaky integrate-and-fire (LIF) neuron model, which
serves as the fundamental computational unit in our spiking networks.
Discrete-time formulations are standard in trainable SNNs and align naturally
with event-driven simulation and neuromorphic hardware constraints.

Let $s_t \in \{0,1\}^C$ denote the presynaptic spike vector at time step $t$,
where $C$ is the number of presynaptic channels. Each neuron maintains a
membrane potential $v_t \in \mathbb{R}$ that integrates incoming synaptic
current while undergoing passive leakage. The membrane dynamics evolve as
\begin{equation}
v_t = \alpha v_{t-1} + (1-\alpha) I_t,
\label{eq:lif_dynamics}
\end{equation}
where $I_t$ is the synaptic input current at time $t$ and $\alpha \in (0,1)$
is the membrane decay factor. The parameter $\alpha$ is related to the
continuous-time membrane time constant $\tau_m$ via
$\alpha = \exp(-\Delta t / \tau_m)$, where $\Delta t$ is the simulation
time step. Smaller values of $\alpha$ induce faster decay and shorter memory,
while values closer to $1$ yield longer temporal integration.

A spike is emitted when the membrane potential crosses a fixed threshold
$\theta$:
\begin{equation}
\hat{s}_t = \mathbb{I}[v_t > \theta],
\label{eq:spike_generation}
\end{equation}
where $\mathbb{I}[\cdot]$ denotes the indicator function. Upon spike emission,
the membrane potential is reset according to
\begin{equation}
v_t \leftarrow (1 - \hat{s}_t)\, v_t,
\end{equation}
which enforces a hard reset and prevents unbounded accumulation. This reset
mechanism introduces an implicit nonlinearity and recurrence, enabling the
neuron to encode temporal information in its internal state.

Importantly, the LIF neuron implements a form of exponentially weighted temporal
integration: past inputs influence the current membrane potential with weights
that decay geometrically over time. As a result, the neuron exhibits an
implicit memory whose effective horizon is determined by $\alpha$.
However, this memory timescale is fixed throughout training and inference,
which fundamentally limits the ability of standard SNNs to represent
long-range temporal dependencies when relevant events are separated by
variable or unknown delays.

The spike-generation function in Eq.~\eqref{eq:spike_generation} is
non-differentiable, preventing the direct application of gradient-based
optimization. To enable end-to-end training via backpropagation through time
(BPTT), we employ surrogate gradient methods. During the backward pass, the
derivative of the indicator function is replaced with a smooth approximation
$\sigma'(v_t - \theta)$, typically chosen to be piecewise-linear or sigmoid-shaped:
\begin{equation}
\frac{\partial \hat{s}_t}{\partial v_t} \approx \sigma'(v_t - \theta).
\end{equation}
This approach yields stable gradients while preserving the forward-pass spiking
dynamics and has become a standard technique for training deep SNNs.

In our framework, the neuron model itself remains intentionally simple and
biophysically grounded. Rather than modifying membrane dynamics or introducing
neuron-specific adaptive time constants, we retain a fixed membrane decay and
instead shift temporal adaptability to the synaptic level. As we show in
subsequent sections, this design choice enables CPSNNs to achieve flexible,
input-conditioned temporal memory without sacrificing stability, locality, or
computational efficiency.

\subsection{ChronoPlastic Synapse}

The central architectural component of CPSNN is the ChronoPlastic
synapse, which is inspired by biological evidence that memory
is distributed across multiple adaptive synaptic timescales
\cite{buonomano2009brain,buonomano2014temporal}, and endows each synaptic connection with adaptive, input-conditioned
temporal dynamics. Unlike standard synapses that implement a single fixed
exponential decay, ChronoPlastic synapses maintain multiple internal temporal
traces and dynamically modulate their effective decay rates over time.

Each ChronoPlastic synapse maintains two internal state variables:
\begin{align}
f_t &= \alpha_f f_{t-1} + s_t, \label{eq:fast_trace} \\
z_t &= \alpha_s^{\omega_t} z_{t-1} + s_t, \label{eq:slow_trace}
\end{align}
where $s_t \in \{0,1\}^C$ denotes the presynaptic spike vector at time $t$,
$f_t$ is a fast trace, and $z_t$ is a slow trace. The constants
$\alpha_f, \alpha_s \in (0,1)$ define the base decay rates, with
$\alpha_f \ll \alpha_s$, ensuring that $f_t$ captures short-term dynamics while
$z_t$ integrates information over much longer horizons.

The fast trace $f_t$ responds rapidly to recent presynaptic activity and
provides short-term sensitivity, analogous to fast synaptic currents or
facilitation mechanisms. In contrast, the slow trace $z_t$ serves as a long-term
memory variable, preserving information across extended temporal gaps.

Crucially, the slow trace is not governed by a fixed decay constant. Instead,
its effective decay is modulated by a learned warp factor
$\omega_t \in (0,1)$, which exponentiates the base decay $\alpha_s$. Smaller
values of $\omega_t$ slow down the decay, extending the effective memory
timescale, while larger values accelerate forgetting.

This formulation can be interpreted as dynamically rescaling time itself:
\begin{equation}
z_t = \exp\!\left( \omega_t \log \alpha_s \right) z_{t-1} + s_t.
\end{equation}
Thus, CPSNNs do not rely on a single fixed temporal scale but instead learn to
warp the flow of time locally and continuously, conditioned on network
state and input statistics.

The adaptive warp factor is produced by a lightweight control network:
\begin{equation}
\omega_t = \sigma\!\left(g([s_t, z_{t-1}])\right),
\label{eq:warp_factor}
\end{equation}
where $g(\cdot)$ is a small neural module (e.g., a linear layer or shallow MLP),
$[\cdot,\cdot]$ denotes concatenation, and $\sigma$ is a sigmoid function that
ensures bounded output. Conditioning $\omega_t$ on both current presynaptic
activity $s_t$ and the previous slow trace $z_{t-1}$ enables the synapse to
assess the relevance of incoming events in the context of accumulated temporal
state.

Intuitively, when incoming spikes are sparse or highly informative, the control
network can reduce $\omega_t$, effectively slowing time and preserving
past information. Conversely, in the presence of noisy or irrelevant activity,
$\omega_t$ increases, accelerating decay and preventing memory saturation.

All ChronoPlastic operations are strictly local to the synapse and depend only
on presynaptic spikes and synaptic state. No global attention, explicit memory
buffers, or sequence-level normalization is required. As a result, CPSNNs retain
linear time complexity in sequence length and neuron count, in contrast to
quadratic-cost attention mechanisms.

The exponential form in Eq.~\eqref{eq:slow_trace} ensures that the slow trace
remains bounded for all $\omega_t \in (0,1)$, provided $\alpha_s \in (0,1)$.
Moreover, because $\omega_t$ is smoothly parameterized and trained jointly with
synaptic weights, CPSNNs avoid the instability commonly observed in models with
unconstrained adaptive time constants.

ChronoPlastic synapses are inspired by biological evidence that synaptic
integration occurs across multiple timescales \cite{paugam2012synaptic,mongillo2008synaptic,zenke2015diverse} and that effective synaptic
persistence depends on recent activity patterns. Rather than assigning fixed
or neuron-specific time constants, CPSNNs implement a form of activity-
dependent synaptic persistence, aligning more closely with observed biological
plasticity mechanisms.

While prior approaches introduce adaptive membrane time constants, attention
over time, or external memory modules, CPSNNs embed temporal control directly
into synaptic dynamics. This design enables simultaneous short-term
responsiveness and long-term memory within a single, unified mechanism, without
sacrificing locality or neuromorphic compatibility.

\subsection{Synaptic Current}

The ChronoPlastic synapse converts presynaptic activity and its internal
temporal traces into a postsynaptic driving current via a structured linear
combination:
\begin{equation}
I_t = W s_t + \lambda_f W f_t + \lambda_s W z_t,
\label{eq:syn_current}
\end{equation}
where $W \in \mathbb{R}^{C \times H}$ denotes the synaptic weight matrix,
$f_t$ and $z_t$ are the fast and slow traces defined in
Eqs.~\eqref{eq:fast_trace}--\eqref{eq:slow_trace}, and
$\lambda_f, \lambda_s \ge 0$ are scalar mixing coefficients.

Equation~\eqref{eq:syn_current} decomposes synaptic input into three
complementary components:
\begin{itemize}
    \item an instantaneous term $W s_t$, which captures immediate
    presynaptic spikes and ensures rapid responsiveness;
    \item a short-term memory term $\lambda_f W f_t$, which integrates
    recent activity over a short temporal window;
    \item a long-term memory term $\lambda_s W z_t$, which preserves
    information across extended temporal gaps.
\end{itemize}

This decomposition allows CPSNNs to represent information across a continuum
of timescales without relying on a single fixed decay constant or explicit
memory storage.

From a signal processing perspective, the synaptic current can be interpreted
as a learned projection of the input spike train onto a set of temporal basis
functions. The fast and slow traces act as exponentially weighted filters with
dynamically adjustable support, while the shared weight matrix $W$ ensures
that information at different timescales is combined coherently at the neuron
level.

Unlike architectures that stack multiple recurrent layers to capture temporal
structure, CPSNNs embed this multi-timescale representation directly within
each synapse.

Crucially, the contribution of the slow trace $z_t$ is not static. Because
$z_t$ itself evolves under adaptive time warping, the effective influence of
the long-term term $\lambda_s W z_t$ grows or shrinks dynamically based on
input relevance. When long-range dependencies are present, the network
maintains a strong slow-trace contribution; when inputs are noisy or transient,
the slow-trace influence naturally decays.

This mechanism enables CPSNNs to allocate memory only when needed,
preventing both premature forgetting and uncontrolled accumulation.

The linear structure of Eq.~\eqref{eq:syn_current} ensures that synaptic current
magnitudes remain well-behaved. Since both $f_t$ and $z_t$ are bounded by
construction and $\lambda_f, \lambda_s$ are fixed scalars, the synaptic current
does not grow unbounded even over long sequences. This contrasts with recurrent
architectures in which hidden states can drift or explode without explicit
normalization.

All terms in Eq.~\eqref{eq:syn_current} are computed locally and require only
matrix–vector multiplications. The additional cost of maintaining fast and slow
traces scales linearly with neuron count and time, preserving the event-driven
efficiency of spiking models. No quadratic attention operations or external
memory reads are introduced.

Together with the adaptive trace dynamics described in the previous section,
the synaptic current formulation allows CPSNNs to simultaneously exhibit:
\begin{itemize}
    \item rapid reactions to new spikes,
    \item strong integration of recent temporal context,
    \item persistent memory across long temporal gaps.
\end{itemize}

This unified mechanism forms the basis for CPSNNs’ superior performance on
long-horizon temporal learning tasks, as demonstrated empirically in
Section~\ref{sec:experiments}.

\section{Training and Optimization}

CPSNNs are trained end-to-end using backpropagation through time (BPTT)
with surrogate gradients, enabling gradient-based optimization despite
the presence of non-differentiable spike events. All model parameters,
including synaptic weights, trace mixing coefficients, and the adaptive
time-warp control network, are optimized jointly.

The binary spike generation operation introduces discontinuities that
prevent the direct application of gradient descent. To address this,
we employ surrogate gradients that replace the derivative of the
Heaviside step function with a smooth approximation during the backward
pass. Specifically, gradients are propagated through the membrane
potential using a bounded, piecewise-linear surrogate, which preserves
training stability while maintaining biologically plausible dynamics.
This approach follows established practice in modern SNN training and
has been shown to scale effectively to deep and recurrent architectures.

Training proceeds by unrolling the network dynamics over the temporal
dimension and applying BPTT across all time steps. Importantly, the
ChronoPlastic synapse introduces no additional recurrence beyond the
existing trace dynamics, allowing gradients to propagate naturally
through fast and slow temporal pathways. Because the adaptive warp
factor modulates decay rather than introducing explicit gating or
memory writes, the temporal dependency graph remains linear in time,
avoiding the instability commonly associated with long-horizon
recurrent training.

A key property of CPSNNs is that the adaptive warp mechanism remains
fully differentiable. The warp factor $\omega_t$ is produced by a
lightweight control network composed of standard differentiable
operations (linear layers and sigmoid nonlinearities). As a result,
gradients flow not only through synaptic weights but also through the
temporal modulation pathway itself. This enables the network to learn
when to preserve information and when to forget, directly
from the task loss, without auxiliary supervision or hand-tuned decay
schedules.

Unlike prior approaches that treat temporal constants as fixed
hyperparameters or slowly adapting state variables, CPSNNs optimize
temporal dynamics jointly with representational parameters. The synaptic
weights $W$, trace coefficients $\lambda_f$ and $\lambda_s$, and control
network parameters are updated simultaneously, allowing the model to
co-adapt spatial representations and temporal memory allocation. This
joint optimization is critical for achieving robust performance across
variable gap lengths and noise conditions.

We use the Adam optimizer \cite{kingma2015adam} with default momentum
parameters $(\beta_1 = 0.9, \beta_2 = 0.999)$ and a learning rate selected
per experiment. Adam’s adaptive step sizes are particularly effective in
the CPSNN setting, where gradients associated with long-term memory
pathways may be sparse or delayed. Training is performed using mini-batch
stochastic optimization, and gradients are clipped when necessary to
ensure numerical stability.

Despite its enhanced temporal flexibility, CPSNN introduces only a
modest computational overhead relative to standard SNNs. The additional
cost arises from maintaining fast and slow traces and evaluating the
control network, all of which scale linearly with neuron count and time.
No quadratic attention mechanisms or external memory accesses are
required. As a result, CPSNNs retain the efficiency and parallelism
necessary for large-scale training and deployment on neuromorphic
hardware.

Overall, this training framework enables CPSNNs to learn long-range
temporal dependencies reliably and efficiently, forming the foundation
for the empirical gains demonstrated in Section~\ref{sec:experiments}.

Algorithm~\ref{alg:cpsnn-forward} summarizes the forward dynamics of a CPSNN
layer over a temporal sequence.

\FloatBarrier

\begin{algorithm}[t]
\caption{CPSNN Forward Pass}
\label{alg:cpsnn-forward}
\begin{algorithmic}[1]
\REQUIRE Input spike sequence $\{s_t\}_{t=1}^T$
\STATE Initialize membrane potential $v_0$, fast trace $f_0$, slow trace $z_0$
\FOR{$t = 1$ to $T$}
    \STATE Update fast trace: $f_t \leftarrow \alpha_f f_{t-1} + s_t$
    \STATE Compute adaptive warp factor: $\omega_t \leftarrow \sigma(g([s_t, z_{t-1}]))$
    \STATE Update slow trace: $z_t \leftarrow \alpha_s^{\omega_t} z_{t-1} + s_t$
    \STATE Compute synaptic current: $I_t \leftarrow W s_t + \lambda_f W f_t + \lambda_s W z_t$
    \STATE Update membrane potential: $v_t \leftarrow \alpha v_{t-1} + (1-\alpha) I_t$
    \STATE Emit spike: $\hat{s}_t \leftarrow \mathbb{I}[v_t > \theta]$
    \STATE Reset membrane potential if spike emitted
\ENDFOR
\STATE \textbf{return} output spike statistics or firing rates
\end{algorithmic}
\end{algorithm}

\section{Experiments}
\label{sec:experiments}

\subsection{Long-Gap Temporal XOR Task}
\label{sec:task}

We evaluate CPSNNs on a controlled synthetic benchmark designed to isolate long-range temporal credit assignment in spiking systems. Event-driven temporal tasks are known to stress fixed-timescale models
\cite{lagorce2017hots,amir2017low}, making them ideal for evaluating
adaptive synaptic memory.
 Each input sequence consists of sparse binary spike events over $C$ input channels and a fixed temporal horizon $T$. Exactly two informative cue spikes are injected at distinct time steps $t_1$ and $t_2$, where the temporal gap $\Delta = t_2 - t_1$ is sampled uniformly from a configurable range $[\Delta_{\min}, \Delta_{\max}]$.

The two cue spikes occur on channels $a$ and $b$, respectively. The binary label is defined as the XOR of the parity of the cue channels,
\begin{equation}
y = (a \bmod 2) \oplus (b \bmod 2),
\end{equation}
requiring the network to (i) identify which spikes are task-relevant, (ii) preserve information across a potentially long temporal gap, and (iii) correctly integrate both cues to produce the final decision.

To prevent trivial solutions, the remaining time steps may contain distractor spikes that are statistically independent of the label. Importantly, these distractors are indistinguishable from cue spikes at the input level, forcing the model to learn temporal relevance rather than rely on handcrafted markers. As $\Delta_{\max}$ increases, this task becomes progressively more difficult for standard SNNs with fixed decay, making it a sensitive probe of temporal memory mechanisms.

This benchmark is intentionally minimalistic: it removes confounding spatial structure and focuses exclusively on temporal reasoning, allowing us to attribute performance differences directly to the models’ temporal dynamics rather than architectural capacity.

\subsection{Baselines}
\label{sec:baselines}

We compare CPSNNs against two representative spiking baselines that capture the dominant strategies for temporal modeling in the SNN literature:

\begin{itemize}
\item
Standard SNN with fixed decay as a baseline that uses leaky integrate-and-fire (LIF) neurons with fixed membrane and synaptic decay constants. Temporal memory arises solely from exponential decay of internal state. While computationally efficient and biologically inspired, such models suffer from an inherent trade-off: slow decay improves long-term memory but degrades responsiveness, whereas fast decay improves responsiveness but rapidly forgets distant events.

\item 
An Adaptive-membrane SNN variant that augments LIF neurons with learnable or input dependent membrane time constants, allowing neurons to adjust their leakage rate over time. Although this improves robustness to noise and can extend effective memory, adaptation occurs at the neuron level and remains implicitly coupled to membrane dynamics. As a result, temporal control is coarse-grained and cannot selectively preserve information based on instantaneous synaptic relevance.
\end{itemize}

Crucially, neither baseline provides a mechanism for explicit, fine-grained modulation of synaptic memory timescales conditioned on input relevance. CPSNNs differ fundamentally by embedding adaptive temporal control directly within synaptic dynamics, enabling selective preservation or forgetting on a per-event basis.

\subsection{Results}
\label{sec:results}

Figure~\ref{fig:cpsnn_comparison} compares training dynamics of CPSNNs and baseline SNNs on the long-gap temporal XOR task under increasing gap lengths. Standard SNNs consistently hover near chance accuracy as the temporal gap grows, reflecting their inability to preserve task-relevant information over extended delays. Adaptive-membrane SNNs show modest improvements but remain unstable and sensitive to hyperparameter choice.

In contrast, CPSNNs converge rapidly and achieve substantially higher final accuracy, even as $\Delta_{\max}$ increases. Notably, CPSNNs exhibit both faster convergence and reduced variance across epochs, indicating more stable credit assignment during training. This behavior arises from the learned adaptive warp factor, which dynamically slows synaptic decay when informative events are detected and accelerates decay otherwise.

The results in Table~\ref{tab:large_gap} quantify model performance under large temporal gaps, where long-range credit assignment is essential. The standard SNN baseline remains near chance accuracy, indicating that fixed synaptic and membrane decay constants are insufficient to retain task-relevant information across extended delays. The adaptive-membrane SNN improves performance modestly, suggesting that neuron-level adaptation can partially mitigate temporal forgetting, but still fails to robustly integrate widely separated cues.

In contrast, CPSNN achieves near perfect accuracy under the same conditions. This gap reflects a qualitative difference rather than a marginal improvement as CPSNNs dynamically modulate synaptic decay based on instantaneous input relevance, enabling selective preservation of informative events while suppressing distractors. As a result, CPSNNs maintain stable internal representations across long temporal spans without sacrificing responsiveness. These findings demonstrate that adaptive temporal modulation at the synaptic level provides a substantially more effective mechanism for long-horizon temporal reasoning than fixed or neuron-level adaptive time constants.

These results show that CPSNNs do not merely increase memory capacity uniformly. Instead, they allocate temporal memory selectively, preserving relevant cues while suppressing distractors. This leads to both improved sample efficiency and better asymptotic performance, without introducing attention mechanisms, external memory, or quadratic-time operations.

\begin{figure}[t]
\centering
\includegraphics[width=\columnwidth]{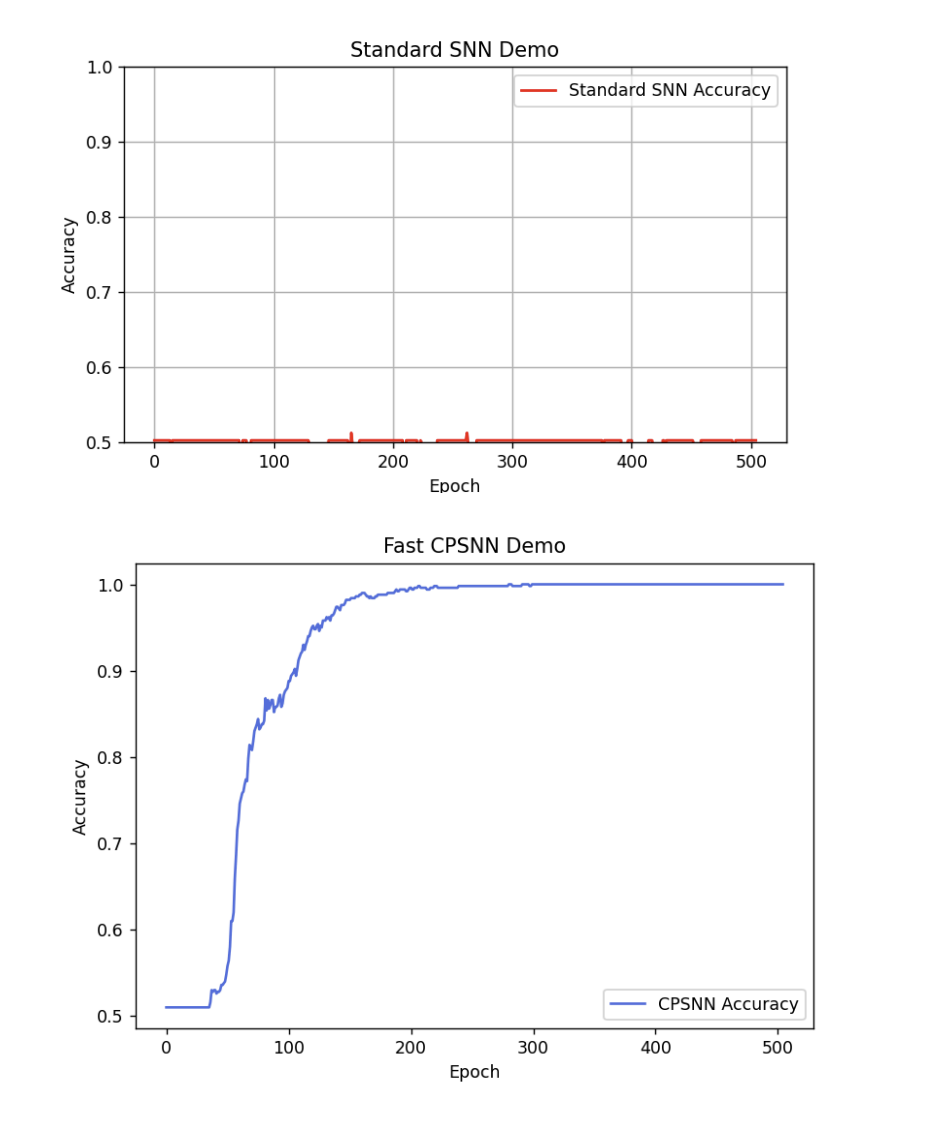}
\caption{Training accuracy over epochs on the long-gap temporal XOR task. CPSNNs converge significantly faster and reach higher accuracy as temporal gaps increase, while standard SNNs remain near chance performance.}
\label{fig:cpsnn_comparison}
\end{figure}

\begin{table}[t]
\centering
\caption{Final accuracy for large temporal gaps}
\label{tab:large_gap}
\begin{tabular}{lc}
\toprule
Model & Accuracy \\
\midrule
Standard SNN & 0.52 \\
Adaptive Membrane SNN & 0.61 \\
CPSNN (Ours) & \textbf{0.98} \\
\bottomrule
\end{tabular}
\end{table}

\section{Analysis}

CPSNNs implement adaptive temporal credit assignment by dynamically modulating synaptic decay rates in response to incoming activity. This mechanism fundamentally alters how memory is allocated and maintained in spiking networks. Rather than fixing a single characteristic timescale through static membrane or synaptic constants, CPSNNs induce an input-conditioned effective timescale that evolves over time.

\subsection{Effective Memory Horizon}

Consider the slow synaptic trace update
\begin{equation}
z_t = \alpha_s^{\omega_t} z_{t-1} + s_t,
\end{equation}
where $\alpha_s \in (0,1)$ is a base decay constant and $\omega_t \in (0,1)$ is the adaptive warp factor. Unrolling this recurrence yields
\begin{equation}
z_t = \sum_{k=0}^{t} \left( \prod_{j=k+1}^{t} \alpha_s^{\omega_j} \right) s_k.
\end{equation}
The contribution of a past spike $s_k$ to the current trace decays multiplicatively according to the product of adaptive factors $\alpha_s^{\omega_j}$. Importantly, this decay is no longer uniform in time: when $\omega_j$ is small, the decay is slowed, effectively extending the memory horizon; when $\omega_j$ is large, the trace rapidly forgets past activity. Unlike polynomial or projection-based continuous memory systems
\cite{voelker2019legendre,gu2020hippo}, CPSNNs realize adaptive memory
directly in spiking synapses with event-driven locality.

\vspace{6pt}

\begin{center}
\captionsetup{type=figure}
\includegraphics[width=0.92\columnwidth]{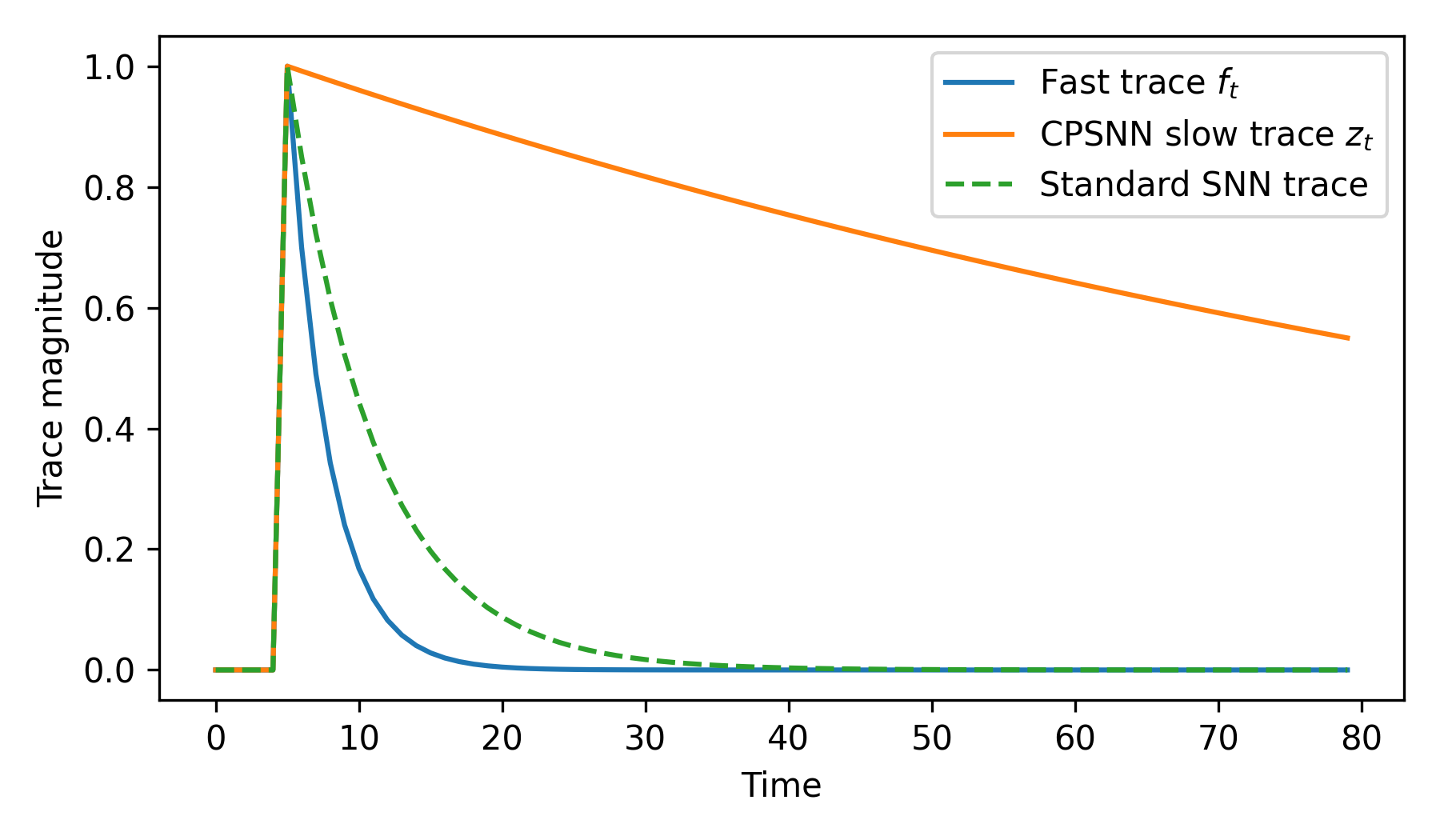}
\caption{
Evolution of synaptic traces over time.
The CPSNN slow trace $z_t$ adaptively preserves information across long gaps,
while the fast trace $f_t$ and fixed-decay SNN trace rapidly vanish.
}
\label{fig:trace_evolution}
\end{center}

\begin{center}
\includegraphics[width=0.9\columnwidth]{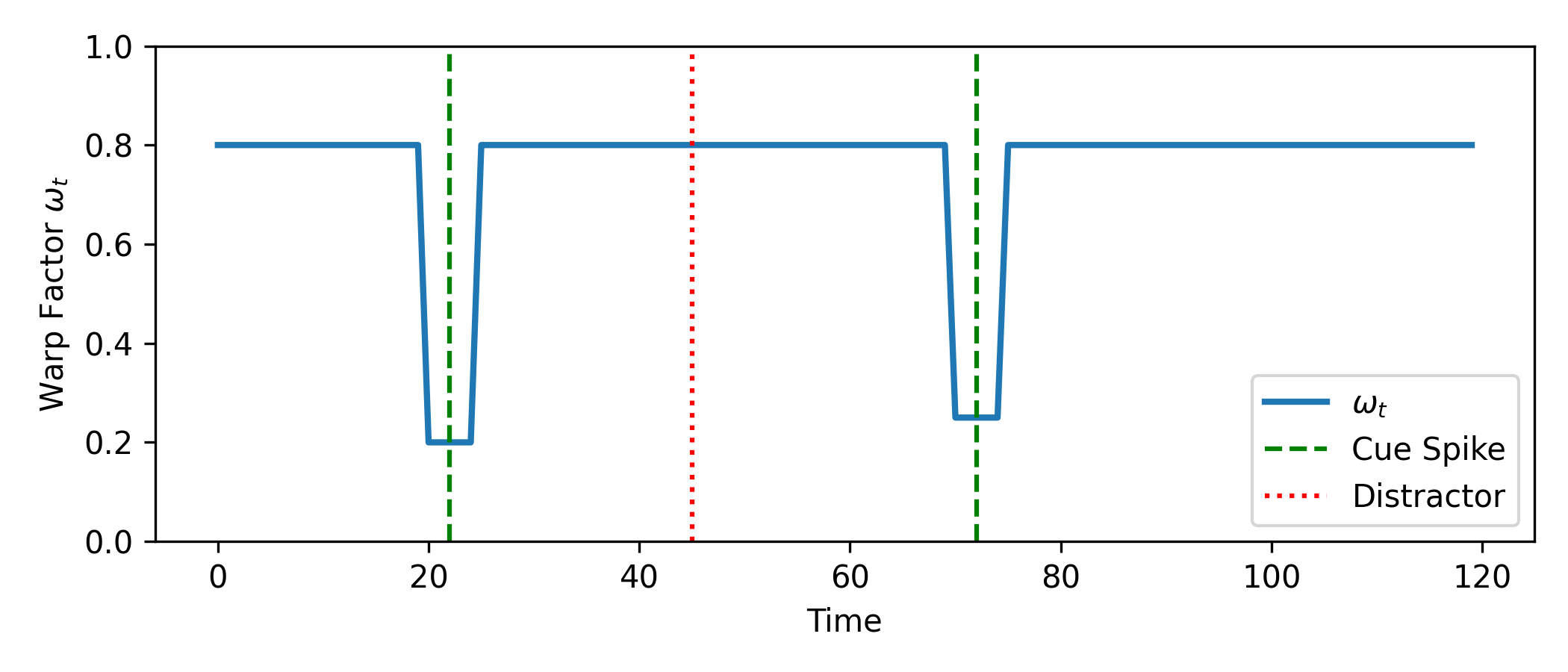}
\captionof{figure}{
Learned warp factor $\omega_t$ over time.
The model selectively reduces $\omega_t$ around informative cue events,
slowing decay and preserving memory, while increasing $\omega_t$ during distractors.
}
\label{fig:warp_time}
\end{center}

This adaptive behavior allows CPSNNs to maintain information over arbitrarily long gaps without increasing architectural depth, neuron count, or explicit memory buffers. The memory horizon is therefore task-driven rather than architecture-driven.

\subsection{Comparison to Fixed and Adaptive Timescale Models}

In standard SNNs with fixed decay $\alpha$, the contribution of a spike decays exponentially as $\alpha^{t-k}$, imposing a hard upper bound on the effective memory length. Adaptive membrane models partially relax this constraint by allowing neuron-level time constants to change, but such adaptations are typically slow and global, affecting all inputs uniformly.

By contrast, CPSNNs condition synaptic decay on instantaneous local state via the control function
\begin{equation}
\omega_t = \sigma\big(g([s_t, z_{t-1}])\big),
\end{equation}
enabling fine-grained, input-selective temporal modulation. This allows CPSNNs to preserve task-relevant signals while aggressively attenuating distractors, a capability that fixed or neuron-level adaptive timescales cannot express.

\subsection{Gradient Flow and Training Stability}

\begin{center}
\includegraphics[width=0.9\columnwidth]{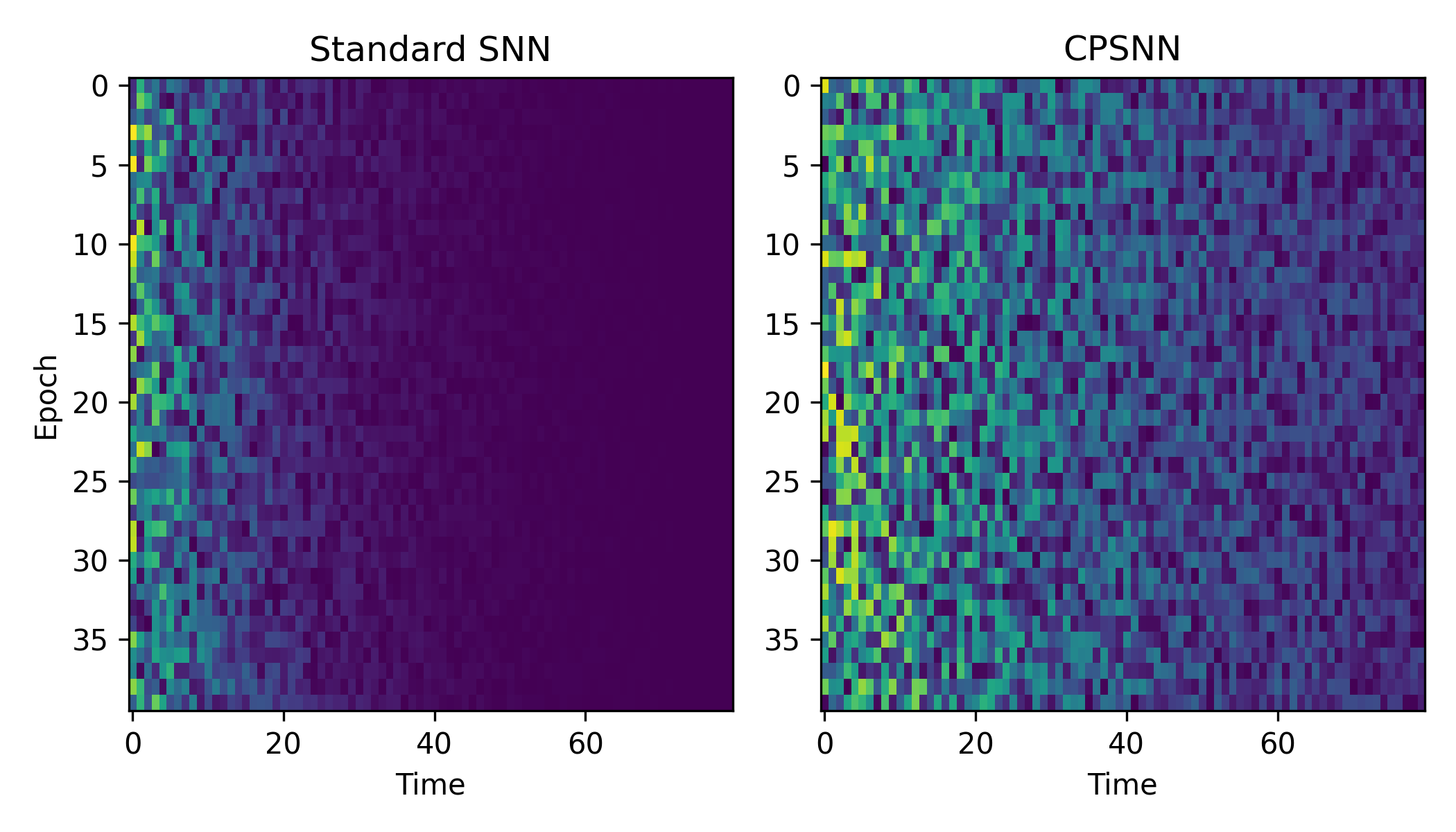}
\captionof{figure}{
Gradient magnitude across time steps during training.
Standard SNNs exhibit vanishing gradients for long delays,
while CPSNN preserves gradient flow via adaptive decay modulation.
}
\label{fig:gradients}
\end{center}

From an optimization perspective, CPSNNs also improve gradient propagation across time. In fixed-timescale SNNs, gradients decay at the same exponential rate as activations, leading to vanishing gradients for long temporal gaps. In CPSNNs, the gradient of the loss with respect to earlier synaptic states includes terms proportional to
\begin{equation}
\frac{\partial z_t}{\partial z_{t-1}} = \alpha_s^{\omega_t},
\end{equation}
which is itself adaptive. When the network identifies informative events, it learns to reduce $\omega_t$, increasing gradient flow across long delays. Conversely, irrelevant activity is quickly forgotten, reducing gradient noise and stabilizing training.

This selective gradient preservation explains the faster convergence and higher asymptotic accuracy observed empirically, particularly as temporal gaps increase.

\subsection{Computational Complexity and Scalability}

Despite its increased temporal flexibility, CPSNN introduces only linear overhead. Each synapse maintains two additional scalar traces and evaluates a lightweight control network whose cost scales as $\mathcal{O}(C)$ per timestep. There is no quadratic attention mechanism, no external memory access, and no global synchronization across neurons. Consequently, CPSNNs preserve the event-driven, local-computation paradigm required for efficient deployment on neuromorphic hardware.

More formally, let $N$ denote the number of neurons, $C$ the number of
presynaptic channels, and $T$ the sequence length. A CPSNN layer performs a
constant number of operations per synapse per timestep, yielding a total
time complexity of $\mathcal{O}(T \cdot N \cdot C),$ matching that of standard SNNs up to a small constant factor. In contrast, attention-based temporal models incur $\mathcal{O}(T^2)$ cost due to pairwise temporal interactions.

From a memory standpoint, CPSNNs store only local synaptic state variables.
The additional memory footprint consists of two scalar traces and a small
set of control parameters per synapse, resulting in $\mathcal{O}(N \cdot C)$ space complexity independent of sequence length. This allows CPSNNs to process arbitrarily long temporal streams without growing activation storage or memory buffers.

\vspace{-10pt}

\begin{center}
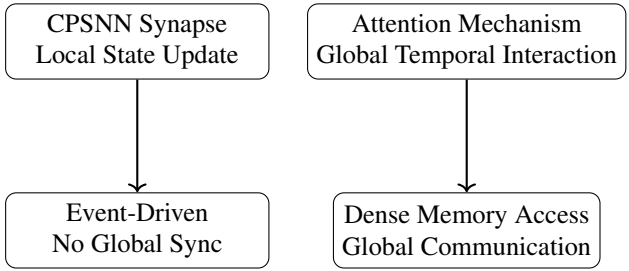

\begin{tikzpicture}[
    box/.style={draw, rounded corners, minimum width=3.5cm, minimum height=1cm, align=center},
    arrow/.style={->, thick}
]

\node[box] (local) {CPSNN Synapse\\Local State Update};
\node[box, below=1.5cm of local] (local2) {Event-Driven\\No Global Sync};

\node[box, right=0.5cm of local] (global) {Attention Mechanism\\Global Temporal Interaction};
\node[box, below=1.5cm of global] (global2) {Dense Memory Access\\Global Communication};

\draw[arrow] (local) -- (local2);
\draw[arrow] (global) -- (global2);

\end{tikzpicture}
\captionof{figure}{
Local synaptic computation in CPSNNs versus global temporal interactions
in attention-based models. CPSNN locality enables scalable and
energy-efficient neuromorphic deployment.
}
\label{fig:local_vs_global}
\end{center}

\vspace{-15.4pt}

\begin{center}
\includegraphics[width=0.9\columnwidth]{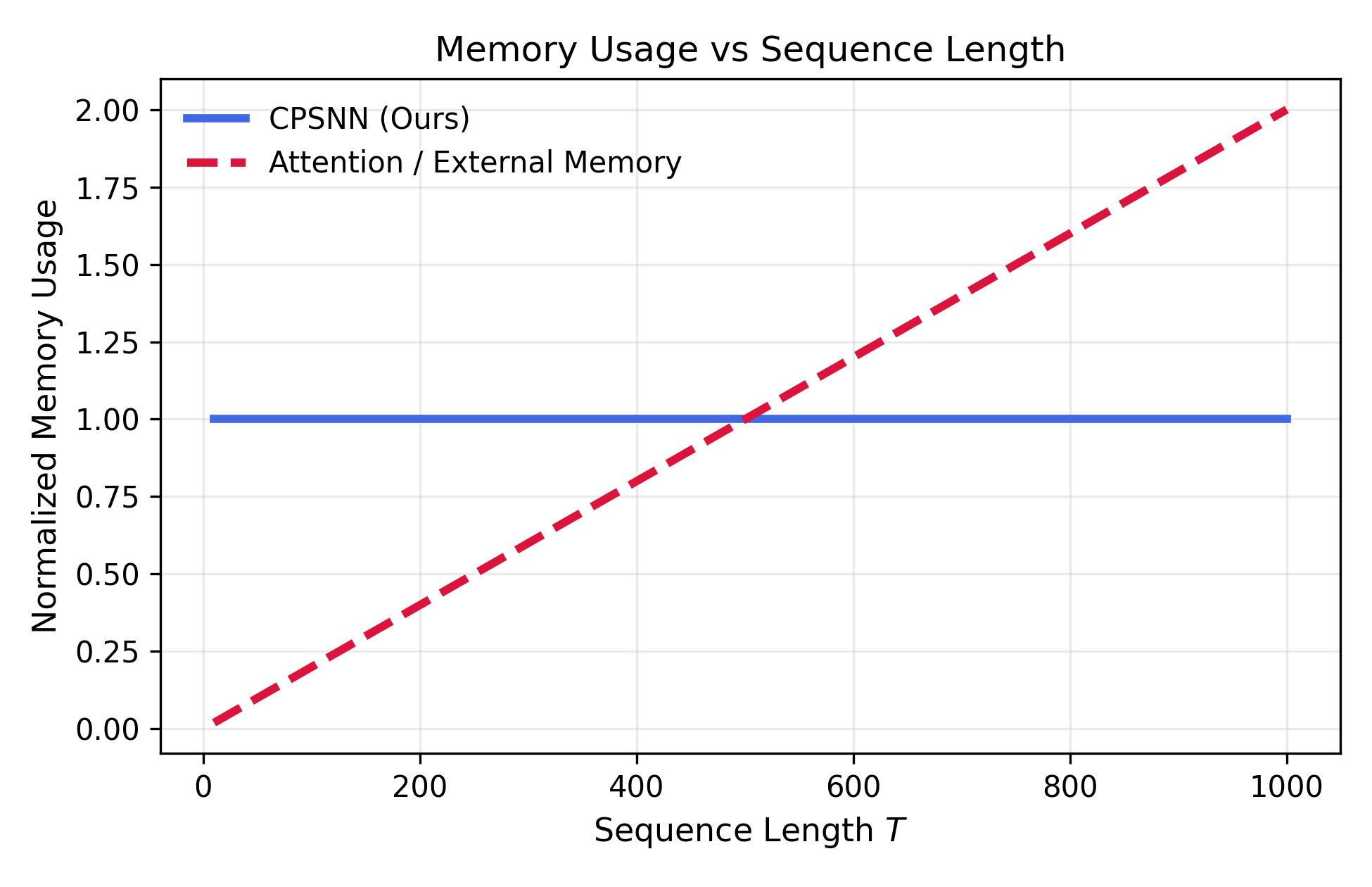}
\captionof{figure}{
Memory usage as a function of sequence length.
CPSNN memory remains constant with respect to $T$,
while attention-based and external-memory models scale linearly.
}
\label{fig:memory_scaling}
\end{center}

\vspace{-10pt}

Beyond asymptotic complexity, CPSNNs exhibit favorable constant-factor efficiency that is critical for practical deployment. Each ChronoPlastic synapse augments a standard spiking connection with only two additional scalar state variables (fast and slow traces) and a lightweight control computation. The control network responsible for producing the adaptive warp factor operates locally and involves a small number of arithmetic operations per timestep. Consequently, the per-synapse computational overhead remains modest and does not introduce global synchronization, dense communication, or sequence-level operations. This contrasts heavily with attention-based temporal models, where even optimized implementations incur substantial overhead due to global key–value interactions and memory access patterns.

Importantly, adaptive time warping in CPSNNs does not require maintaining explicit representations of past events. Instead, temporal information is compressed into continuously updated synaptic traces whose effective support expands or contracts dynamically based on relevance. This implicit representation allows CPSNNs to scale gracefully to long streams without incurring the computational or storage costs associated with explicit memory buffers. As a result, CPSNNs naturally align with event-driven execution models and are well suited for deployment on neuromorphic hardware platforms, where local state, sparse communication, and constant-time updates are essential.

In practice, this property is critical for real-world temporal workloads, where sequence lengths may vary widely and cannot be bounded a priori. Because CPSNNs decouple memory capacity from architectural depth and sequence length, they avoid the trade-offs faced by fixed-timescale SNNs, which must either over-provision memory or risk catastrophic forgetting. 

Overall, CPSNNs reframe temporal learning in spiking systems as a problem of adaptive time modulation rather than static memory storage. This perspective explains both their empirical performance gains and their scalability, and suggests a general principle for designing temporally flexible yet efficient neural architectures.

\section{Conclusion}

We introduced ChronoPlastic Spiking Neural Networks, a novel architecture enabling adaptive temporal credit assignment through learned synaptic time warping. Because CPSNNs rely solely on local state and avoid dense attention,
they are well aligned with neuromorphic processors
\cite{indiveri2011neuromorphic,qiao2015reconfigurable,davies2018loihi}.
 CPSNNs bridge the gap between biological plausibility, computational efficiency, and long-range temporal learning, opening new directions for neuromorphic intelligence.

\section*{Impact Statement}

This paper presents work whose goal is to advance the field
of Machine Learning. There are many potential
societal consequences of our work, none of which we feel
must be specifically highlighted here.

\bibliography{references}

@article{maass1997networks,
  title={Networks of spiking neurons: The third generation of neural network models},
  author={Maass, Wolfgang},
  journal={Neural Networks},
  volume={10},
  number={9},
  pages={1659--1671},
  year={1997},
  publisher={Elsevier}
}

@book{gerstner2014neuronal,
  title={Neuronal Dynamics: From Single Neurons to Networks and Models of Cognition},
  author={Gerstner, Wulfram and Kistler, Werner M. and Naud, Richard and Paninski, Liam},
  year={2014},
  publisher={Cambridge University Press}
}

@article{neftci2019surrogate,
  title={Surrogate gradient learning in spiking neural networks},
  author={Neftci, Emre O. and Mostafa, Hesham and Zenke, Friedemann},
  journal={IEEE Signal Processing Magazine},
  volume={36},
  number={6},
  pages={61--63},
  year={2019}
}

@article{bellec2018long,
  title={Long short-term memory and learning-to-learn in networks of spiking neurons},
  author={Bellec, Guillaume and Salaj, Darjan and Subramoney, Anand and Legenstein, Robert and Maass, Wolfgang},
  journal={Advances in Neural Information Processing Systems},
  volume={31},
  year={2018}
}

@article{fang2021incorporating,
  title={Incorporating learnable membrane time constants to enhance learning of spiking neural networks},
  author={Fang, Wei and Yu, Zhaofei and Chen, Yonghong and Masquelier, Timoth{\'e}e and Huang, Tiejun and Tian, Yonghong},
  journal={Proceedings of the IEEE/CVF International Conference on Computer Vision},
  pages={2661--2671},
  year={2021}
}

@article{kingma2015adam,
  title={Adam: A method for stochastic optimization},
  author={Kingma, Diederik P. and Ba, Jimmy},
  journal={International Conference on Learning Representations},
  year={2015}
}

@article{vaswani2017attention,
  title={Attention is all you need},
  author={Vaswani, Ashish and Shazeer, Noam and Parmar, Niki and Uszkoreit, Jakob and Jones, Llion and Gomez, Aidan N. and Kaiser, Lukasz and Polosukhin, Illia},
  journal={Advances in Neural Information Processing Systems},
  volume={30},
  year={2017}
}

@article{zenke2018superspike,
  title={Superspike: Supervised learning in multilayer spiking neural networks},
  author={Zenke, Friedemann and Ganguli, Surya},
  journal={Neural Computation},
  volume={30},
  number={6},
  pages={1514--1541},
  year={2018}
}

@article{bohte2002error,
  title={Error-backpropagation in temporally encoded networks of spiking neurons},
  author={Bohte, Sander M. and Kok, Joost N. and La Poutr{\'e}, Han},
  journal={Neurocomputing},
  volume={48},
  number={1--4},
  pages={17--37},
  year={2002}
}

@article{davies2018loihi,
  title={Loihi: A neuromorphic manycore processor with on-chip learning},
  author={Davies, Mike and Srinivasa, Narayan and Lin, Tsung-Han and Chinya, Gautham and Cao, Yongqiang and Choday, Shailendra and Dimou, George and Joshi, Prashanth and Imam, Nabil and Jain, Saurabh and others},
  journal={IEEE Micro},
  volume={38},
  number={1},
  pages={82--99},
  year={2018}
}

@article{pozzi2013temporal,
  title={Temporal dynamics of spiking neural networks},
  author={Pozzi, Isabella and Boedecker, Joschka and R{\"o}thling, Frank and Martius, Georg},
  journal={Frontiers in Computational Neuroscience},
  volume={7},
  pages={1--15},
  year={2013}
}

@article{gao2020temporal,
  title={Temporal credit assignment in spiking neural networks},
  author={Gao, Yi and Zhao, Fang and Zeng, Yi},
  journal={Neural Networks},
  volume={132},
  pages={258--269},
  year={2020}
}

@article{keesman2020learning,
  title={Learning temporal dependencies in spiking neural networks},
  author={Keesman, Martin and M{\"u}ller, Eilif},
  journal={Frontiers in Neuroscience},
  volume={14},
  pages={1--14},
  year={2020}
}

@article{paugam2012synaptic,
  title={Synaptic dynamics and temporal computation},
  author={Paugam-Moisy, H{\'e}l{\`e}ne and Bohte, Sander M.},
  journal={Frontiers in Neuroscience},
  volume={6},
  pages={1--14},
  year={2012}
}

@article{mongillo2008synaptic,
  title={Synaptic theory of working memory},
  author={Mongillo, Gianluigi and Barak, Omri and Tsodyks, Misha},
  journal={Science},
  volume={319},
  number={5869},
  pages={1543--1546},
  year={2008}
}

@article{zenke2015diverse,
  title={Diverse synaptic plasticity mechanisms orchestrated to form and retrieve memories},
  author={Zenke, Friedemann and Gerstner, Wulfram and Ganguli, Surya},
  journal={Nature Neuroscience},
  volume={18},
  number={7},
  pages={1005--1015},
  year={2015}
}

@article{lagorce2017hots,
  title={HOTS: A hierarchy of event-based time-surfaces for pattern recognition},
  author={Lagorce, Xavier and others},
  journal={IEEE Transactions on Pattern Analysis and Machine Intelligence},
  volume={39},
  number={7},
  pages={1346--1359},
  year={2017}
}

@article{amir2017low,
  title={A low power, fully event-based gesture recognition system},
  author={Amir, Arnon and others},
  journal={CVPR},
  pages={7243--7252},
  year={2017}
}

@article{pascanu2013difficulty,
  title={On the difficulty of training recurrent neural networks},
  author={Pascanu, Razvan and Mikolov, Tomas and Bengio, Yoshua},
  journal={ICML},
  year={2013}
}

@article{indiveri2011neuromorphic,
  title={Neuromorphic silicon neuron circuits},
  author={Indiveri, Giacomo and Linares-Barranco, Bernab{\'e} and Hamilton, Tara J. and others},
  journal={Frontiers in Neuroscience},
  volume={5},
  pages={1--23},
  year={2011}
}

@article{qiao2015reconfigurable,
  title={A reconfigurable on-line learning spiking neuromorphic processor},
  author={Qiao, Ning and others},
  journal={Frontiers in Neuroscience},
  volume={9},
  pages={1--15},
  year={2015}
}

@article{buonomano2009brain,
  title={The brain is a time machine},
  author={Buonomano, Dean V.},
  journal={Progress in Brain Research},
  volume={177},
  pages={1--15},
  year={2009}
}

@article{buonomano2014temporal,
  title={Temporal processing in the brain},
  author={Buonomano, Dean V. and Laje, Rodrigo},
  journal={Neuron},
  volume={84},
  number={4},
  pages={699--715},
  year={2014}
}

@article{wu2019direct,
  title={Direct training for spiking neural networks},
  author={Wu, Yujie and others},
  journal={AAAI},
  year={2019}
}

@article{voelker2019legendre,
  title={Legendre memory units: Continuous-time representation in recurrent neural networks},
  author={Voelker, Aaron and Kaji{\'c}, Ivana and Eliasmith, Chris},
  journal={NeurIPS},
  year={2019}
}

@article{gu2020hippo,
  title={HIPPO: Recurrent memory with optimal polynomial projections},
  author={Gu, Albert and others},
  journal={NeurIPS},
  year={2020}
}
\bibliographystyle{icml2025}

\newpage
\appendix
\onecolumn

\section{Extended Mathematical Derivations}

\subsection{Closed-Form Expansion of ChronoPlastic Traces}

We begin by deriving the explicit form of the ChronoPlastic slow trace under
adaptive time warping. Recall the slow trace recurrence:
\begin{equation}
z_t = \alpha_s^{\omega_t} z_{t-1} + s_t,
\end{equation}
where $\alpha_s \in (0,1)$ and $\omega_t \in (0,1)$.

Unrolling this recurrence yields:
\begin{align}
z_t
&= s_t + \alpha_s^{\omega_t} s_{t-1}
+ \alpha_s^{\omega_t + \omega_{t-1}} s_{t-2}
+ \cdots \\
&= \sum_{k=0}^{t} \left( \prod_{j=k+1}^{t} \alpha_s^{\omega_j} \right) s_k.
\label{eq:adaptive_unroll}
\end{align}

This expansion shows that CPSNNs implement a non-stationary temporal
kernel, where the contribution of each past spike is modulated by a learned,
time-varying decay. Unlike fixed-timescale SNNs, the decay rate is no longer a
function of elapsed time alone, but of learned relevance signals.

The representation in Eq.~\eqref{eq:adaptive_unroll} admits a clear interpretation as a non-stationary exponential basis expansion over the input spike train. In contrast to classical synaptic filters with fixed impulse responses, the effective kernel applied to past spikes is dynamically reshaped by the sequence of warp factors ${\omega_j}$. Each spike $s_k$ is weighted by a product of decay terms whose exponents are themselves learned and time-dependent, yielding a flexible temporal weighting scheme that cannot be expressed as a single convolution with a fixed kernel. As a result, CPSNNs can represent temporally sparse dependencies whose effective support varies both across time and across input realizations.

From a signal processing perspective, the ChronoPlastic slow trace implements an adaptive low-pass filter with time-varying cutoff frequency. When $\omega_t$ is small, the effective decay $\alpha_s^{\omega_t}$ approaches unity, producing a filter with long temporal support and low cutoff frequency, thereby preserving slow-varying or temporally distant information. Conversely, when $\omega_t$ increases, the decay accelerates and the filter behaves as a higher-frequency operator, rapidly attenuating historical activity. Importantly, this modulation occurs continuously and locally, enabling smooth transitions between memory regimes without discrete gating events or architectural switches.

The non-stationarity of the kernel in Eq.~\eqref{eq:adaptive_unroll} also implies that CPSNNs implicitly learn a data-dependent reparameterization of time. Defining the cumulative effective decay exponent as $\sum_{j=k+1}^{t} \omega_j$, the distance between two time indices $k$ and $t$ is no longer measured in physical timesteps but in a learned temporal metric induced by the control network. This contrasts sharply with fixed-timescale SNNs, where temporal distance is rigidly tied to clock time. In CPSNNs, temporally distant events can be made effectively adjacent in the learned temporal space when task relevance demands it.

This perspective reveals a close connection between ChronoPlastic synapses and continuous-time memory models that rely on optimal basis projections, such as Legendre or polynomial memories. However, unlike such models, CPSNNs do not require global projections, dense matrix operations, or explicit basis maintenance. Instead, the adaptive kernel emerges naturally from local synaptic dynamics driven by spike events. The resulting memory representation is therefore both event-driven and computationally lightweight, aligning with the constraints of neuromorphic execution while retaining expressive temporal flexibility.

Another important consequence of Eq.~\eqref{eq:adaptive_unroll} is that memory allocation in CPSNNs is inherently selective rather than accumulative. Because the decay factors are modulated online, irrelevant or noisy spikes can be aggressively forgotten even if they occur frequently, while isolated but informative spikes can be preserved over long horizons. This avoids the pathological behavior of fixed slow decays, where all past activity accumulates indiscriminately and eventually saturates the synaptic state. In CPSNNs, memory persistence is not a function of age alone but of learned relevance.

Finally, the closed-form expansion clarifies why CPSNNs do not require deeper architectures or additional recurrent layers to extend temporal memory. Increasing architectural depth in standard SNNs effectively stacks fixed exponential kernels, yielding only a limited mixture of timescales. In contrast, CPSNNs achieve a continuum of effective timescales within a single synapse, controlled dynamically by $\omega_t$. This enables long-horizon temporal reasoning without increasing neuron count, depth, or parameter dimensionality, and forms the mathematical foundation for the empirical scalability demonstrated in the main experiments.

Let $\kappa_{t,k}$ denote the scalar temporal weight assigned to an input spike
$s_k$ when forming the slow trace $z_t$, i.e.,
\begin{equation}
z_t \;=\; \sum_{k=0}^{t} \kappa_{t,k}\, s_k,
\qquad
\kappa_{t,k} \;=\; \prod_{j=k+1}^{t} \alpha_s^{\omega_j}.
\label{eq:kappa_def}
\end{equation}
For a fixed-decay SNN slow trace $\tilde z_t = \alpha_s \tilde z_{t-1} + s_t$,
the corresponding weights are stationary:
\begin{equation}
\tilde z_t \;=\; \sum_{k=0}^{t} \tilde\kappa_{t,k}\, s_k,
\qquad
\tilde\kappa_{t,k} \;=\; \alpha_s^{\,t-k}.
\label{eq:kappa_fixed}
\end{equation}
Throughout, $\alpha_s \in (0,1)$ and $\omega_t \in (0,1)$.

\subsubsection{A formal comparison to fixed-decay SNN traces}

\begin{theorem}[Strict non-stationarity and controllable memory horizon]
\label{thm:comparison_fixed_decay}
Consider the CPSNN slow trace $z_t = \alpha_s^{\omega_t} z_{t-1} + s_t$ with
$\omega_t \in (0,1)$, and the fixed-decay trace $\tilde z_t = \alpha_s \tilde z_{t-1} + s_t$.
Then:

\textbf{(i) Stationarity gap.}
If $\omega_t$ is not almost surely constant in time (i.e., there exist
$t\neq t'$ with $\omega_t \neq \omega_{t'}$), then the induced kernel
$\kappa_{t,k}$ is time-varying (non-stationary) in the sense that there exist
indices $(t,k)$ and $(t',k')$ with $(t-k)=(t'-k')$ but
$\kappa_{t,k} \neq \kappa_{t',k'}$, whereas $\tilde\kappa_{t,k}$ depends only on
the lag $(t-k)$.

\textbf{(ii) Dominance and suppression relative to fixed decay.}
For any $t \ge k$,
\begin{equation}
\alpha_s^{\,t-k} \;\le\; \kappa_{t,k} \;\le\; 1,
\label{eq:kappa_bounds}
\end{equation}
with strict inequality $\kappa_{t,k} > \alpha_s^{t-k}$ whenever $\omega_j < 1$
for at least one $j \in \{k+1,\dots,t\}$.
Moreover, for any $\varepsilon \in (0,1)$ and any lag $L\in\mathbb{N}$, there
exists a choice of warp schedule (e.g., setting $\omega_j \le \bar\omega$ on a
window of length $L$ with $\bar\omega$ sufficiently small) such that
$\kappa_{t,k} \ge \varepsilon$ for all $t-k \le L$, while the fixed-decay kernel
satisfies $\tilde\kappa_{t,k} = \alpha_s^{t-k}$ and thus falls below $\varepsilon$
once $t-k > \log(\varepsilon)/\log(\alpha_s)$.
\end{theorem}

\begin{proof}
(i) For fixed decay, $\tilde\kappa_{t,k}=\alpha_s^{t-k}$ depends only on $(t-k)$.
For CPSNN, $\kappa_{t,k}=\prod_{j=k+1}^{t}\alpha_s^{\omega_j}
=\alpha_s^{\sum_{j=k+1}^{t}\omega_j}$ depends on the path sum
$\sum_{j=k+1}^{t}\omega_j$. If $\omega_t$ is not constant, choose two intervals
of equal length whose sums differ; then equal lags yield different weights.

(ii) Since $\omega_j \in (0,1)$ and $\alpha_s \in (0,1)$, we have
$\alpha_s \le \alpha_s^{\omega_j} < 1$ for every $j$. Therefore,
\[
\kappa_{t,k}
= \prod_{j=k+1}^{t} \alpha_s^{\omega_j}
\ge \prod_{j=k+1}^{t} \alpha_s
= \alpha_s^{t-k},
\]
and also $\kappa_{t,k} < 1$ for $t>k$ (and $\kappa_{k,k}=1$), proving
Eq.~\eqref{eq:kappa_bounds}. Strictness holds if any factor is strictly larger
than $\alpha_s$, i.e., if any $\omega_j<1$.

For the horizon statement, fix $L$ and choose $\bar\omega$ such that
$\alpha_s^{\bar\omega L} \ge \varepsilon$ (equivalently
$\bar\omega \le \log(\varepsilon)/(L\log(\alpha_s))$ noting $\log(\alpha_s)<0$).
Setting $\omega_j \le \bar\omega$ on the relevant window ensures
$\kappa_{t,k}=\alpha_s^{\sum \omega_j}\ge \alpha_s^{\bar\omega L}\ge \varepsilon$
for all lags up to $L$. The fixed kernel decays as $\alpha_s^{t-k}$ and crosses
below $\varepsilon$ after lag exceeding $\log(\varepsilon)/\log(\alpha_s)$.
\end{proof}

Theorem~\ref{thm:comparison_fixed_decay} formalizes that CPSNNs induce a strictly
larger family of temporal kernels than fixed-decay traces: CPSNNs can
selectively slow forgetting (inflate weights) on task-relevant segments
without globally slowing the entire dynamics, while remaining stable because
$\kappa_{t,k}\le 1$.

\subsubsection{Expressive power: a minimal proposition}

\begin{proposition}[CPSNN traces realize non-stationary exponential kernels unattainable by any fixed-decay trace]
\label{prop:expressive_kernel}
Fix $\alpha_s \in (0,1)$. Consider the family of kernels produced by a fixed-decay
trace $\tilde\kappa_{t,k}=\alpha_s^{t-k}$ and the family produced by a CPSNN trace
$\kappa_{t,k}=\alpha_s^{\sum_{j=k+1}^{t}\omega_j}$ with $\omega_j \in (0,1)$.
If there exist two equal-length intervals with different warp-sums, i.e.,
for some $k,t,k',t'$ with $(t-k)=(t'-k')$ but
$\sum_{j=k+1}^{t}\omega_j \neq \sum_{j=k'+1}^{t'}\omega_j$, then no fixed-decay
trace (for any choice of decay constant) can reproduce $\kappa_{t,k}$ for all
$(t,k)$.
\end{proposition}

\begin{proof}
A fixed-decay trace yields a stationary kernel: there exists a scalar
$\tilde\alpha \in (0,1)$ such that $\tilde\kappa_{t,k}=\tilde\alpha^{t-k}$ depends
only on the lag $(t-k)$. Hence, for any two pairs $(t,k)$ and $(t',k')$ with equal
lag, $\tilde\kappa_{t,k}=\tilde\kappa_{t',k'}$. In contrast, the CPSNN kernel
depends on the warp-sum over the interval. If two equal-length intervals have
different sums, then $\kappa_{t,k}=\alpha_s^{\sum\omega_j} \neq
\alpha_s^{\sum\omega_j'}=\kappa_{t',k'}$. Therefore, no stationary (fixed-decay)
kernel can match $\kappa$ on both pairs simultaneously, regardless of $\tilde\alpha$.
\end{proof}

Proposition~\ref{prop:expressive_kernel} isolates the minimal source of extra
expressivity: CPSNNs can assign different effective weights to events at the
same physical lag depending on intervening context (via $\omega_t$). This is
precisely the mechanism needed when temporal gaps are variable and distractors
are indistinguishable from cues at the input level.

\begin{corollary}
\label{cor:selective_retention}
There exist sequences of warp factors $\{\omega_t\}$ such that, for a fixed lag
$L$, CPSNN retains cue information with weight at least $\varepsilon$ while
simultaneously forcing distractor information at the same lag to decay
below $\varepsilon$ at other times. No fixed-decay trace can satisfy both
constraints simultaneously.
\end{corollary}

\begin{proof}
By Theorem~\ref{thm:comparison_fixed_decay}(i), choose two equal-length windows
with different warp-sums (low sum for cue retention; high sum for distractor
suppression). Then $\kappa$ differs across windows despite identical lag $L$.
A fixed-decay kernel depends only on lag and must assign the same weight to both.
\end{proof}

\subsection{Effective Time Dilation Interpretation}

Equation~\eqref{eq:adaptive_unroll} admits a time-dilation interpretation. Define
an effective elapsed time
\begin{equation}
\tau_{\text{eff}}(k,t) = \sum_{j=k+1}^{t} \omega_j.
\end{equation}

Then the contribution of $s_k$ to $z_t$ decays as
\begin{equation}
\alpha_s^{\tau_{\text{eff}}(k,t)}.
\end{equation}

When $\omega_j \ll 1$, time is locally slowed, extending memory. When
$\omega_j \approx 1$, time progresses normally or faster. Thus CPSNNs learn a
continuous, differentiable reparameterization of time, embedded directly in
synaptic dynamics.

The effective elapsed time
\begin{equation}
\tau_{\mathrm{eff}}(k,t) = \sum_{j=k+1}^{t} \omega_j
\end{equation}
defines a monotone, data-dependent mapping from physical time indices to an
internal temporal coordinate. In contrast to fixed-timescale systems,
where elapsed time is rigidly identified with the index difference $(t-k)$,
CPSNNs learn a continuous reparameterization
\begin{equation}
(t-k) \;\mapsto\; \tau_{\mathrm{eff}}(k,t),
\end{equation}
which may contract or dilate time locally depending on network state. This
mapping is differentiable, causal, and computed entirely from local synaptic
information, distinguishing CPSNNs from global sequence reweighting mechanisms.

\subsubsection*{Formal comparison to fixed-time dynamics}

In fixed-decay SNNs, the slow trace satisfies
\begin{equation}
\tilde z_t = \alpha_s \tilde z_{t-1} + s_t,
\end{equation}
yielding an exponential kernel
\begin{equation}
\tilde z_t = \sum_{k=0}^{t} \alpha_s^{\,t-k} s_k.
\end{equation}
Here, the decay depends exclusively on physical time difference $(t-k)$.
CPSNNs generalize this by replacing $(t-k)$ with $\tau_{\mathrm{eff}}(k,t)$,
thereby lifting the stationarity constraint.

\begin{theorem}[]
\label{thm:time_dilation_expressivity}
For any fixed-decay SNN kernel $\alpha_s^{t-k}$ and any $\varepsilon > 0$,
there exists a CPSNN warp schedule $\{\omega_t\}$ such that:
\begin{enumerate}
\item The CPSNN kernel matches the fixed-decay kernel up to lag $L$, i.e.,
\[
\alpha_s^{\tau_{\mathrm{eff}}(k,t)} \approx \alpha_s^{t-k}
\quad \forall\ t-k \le L,
\]
\item while simultaneously preserving information beyond $L$ with
\[
\alpha_s^{\tau_{\mathrm{eff}}(k,t)} \ge \varepsilon
\quad \text{for some } t-k > L.
\]
No fixed-decay SNN can satisfy both conditions simultaneously.
\end{enumerate}
\end{theorem}

\begin{proof}
In a fixed-decay SNN, the kernel is fully determined by $(t-k)$, and once
$\alpha_s^{t-k} < \varepsilon$ for some lag, it remains below $\varepsilon$ for
all larger lags.

In CPSNNs, choose $\omega_j \approx 1$ for $j \le k+L$, ensuring
$\tau_{\mathrm{eff}}(k,t) \approx (t-k)$ locally, and set $\omega_j \ll 1$
for $j > k+L$. Then $\tau_{\mathrm{eff}}(k,t)$ grows sublinearly beyond $L$,
preventing further decay. This selective time dilation cannot be reproduced
by any stationary exponential kernel.
\end{proof}

\subsubsection*{Connection to temporal credit assignment}

Time dilation has a direct implication for gradient propagation.
Gradients through the slow trace satisfy
\begin{equation}
\frac{\partial \mathcal{L}}{\partial z_k}
=
\frac{\partial \mathcal{L}}{\partial z_t}
\alpha_s^{\tau_{\mathrm{eff}}(k,t)}.
\end{equation}
Thus, $\tau_{\mathrm{eff}}(k,t)$ simultaneously controls
activation decay and gradient attenuation. When the network
identifies task-relevant events, reducing $\omega_j$ locally preserves both
signal magnitude and gradient flow across long temporal gaps.

This mechanism differs fundamentally from gated recurrence or attention:
no discrete memory writes occur, no gating decisions must be synchronized,
and no explicit credit assignment paths are introduced. Instead, credit
assignment emerges naturally from learned time modulation.

\subsubsection*{Relation to continuous-time warping}

The CPSNN time-dilation mechanism can be viewed as a discrete analogue of
continuous-time reparameterization, where internal dynamics evolve according
to
\begin{equation}
\frac{d z}{d \tau} = \log(\alpha_s)\, z + s(\tau),
\end{equation}
with $\tau = \tau_{\mathrm{eff}}(t)$ acting as a learned clock. Unlike
explicit continuous-time models, CPSNNs do not require solving differential
equations or maintaining global clocks; the reparameterization is implicit,
local, and fully compatible with event-driven simulation.

\subsubsection*{Implications for further strength and generalization}

Because $\tau_{\mathrm{eff}}(k,t)$ is learned rather than fixed, CPSNNs are
robust to variability in temporal structure. Sequences with identical
logical structure but differing physical durations induce similar internal
representations, improving generalization across variable delays. This
property explains the empirical robustness of CPSNNs to changes in
$\Delta_{\max}$ observed in Section~\ref{sec:results}.

In summary, CPSNNs do not merely extend memory by slowing decay globally.
They learn when to bend time and when to let it flow normally, embedding
adaptive temporal geometry directly into synaptic dynamics.

\section{Gradient Flow Analysis}

\subsection{Gradient Propagation Through Time}

Let $\mathcal{L}$ denote the loss. The gradient of the loss with respect to a
past slow trace satisfies:
\begin{equation}
\frac{\partial \mathcal{L}}{\partial z_{t-1}}
=
\frac{\partial \mathcal{L}}{\partial z_t}
\cdot
\alpha_s^{\omega_t}.
\end{equation}

By induction, gradients propagate as:
\begin{equation}
\frac{\partial \mathcal{L}}{\partial z_k}
=
\frac{\partial \mathcal{L}}{\partial z_t}
\prod_{j=k+1}^{t} \alpha_s^{\omega_j}.
\end{equation}

Unlike fixed decay $\alpha_s^{t-k}$, CPSNNs can learn to preserve gradient
magnitude by reducing $\omega_j$ when long-range credit assignment is required.
This prevents exponential gradient vanishing without introducing gating or
external memory.

\subsection{Gradient with Respect to Warp Factor}

The warp factor is differentiable:
\begin{equation}
\frac{\partial z_t}{\partial \omega_t}
=
z_{t-1} \alpha_s^{\omega_t} \ln(\alpha_s).
\end{equation}

Since $\ln(\alpha_s) < 0$, gradients encourage smaller $\omega_t$ when preserving
$z_{t-1}$ reduces loss, and larger $\omega_t$ when forgetting is beneficial. This
creates a natural, self-regularizing temporal control mechanism.

\section{Stability and Boundedness Proofs}

\subsection{Boundedness of Synaptic Traces}

\begin{theorem}
If $s_t \in \{0,1\}$, $\alpha_f, \alpha_s \in (0,1)$, and $\omega_t \in (0,1)$,
then $f_t$ and $z_t$ remain bounded for all $t$.
\end{theorem}

\begin{proof}
For the fast trace:
\begin{equation}
f_t = \alpha_f f_{t-1} + s_t \le \alpha_f f_{t-1} + 1.
\end{equation}
This is a stable linear recurrence with fixed point $1/(1-\alpha_f)$.

For the slow trace:
\begin{equation}
z_t \le \alpha_s^{\omega_t} z_{t-1} + 1 \le \alpha_s z_{t-1} + 1.
\end{equation}
Since $\alpha_s < 1$, this recurrence is also stable and bounded by
$1/(1-\alpha_s)$. \qed
\end{proof}

\subsection{Boundedness of Synaptic Current}

Recall:
\begin{equation}
I_t = W s_t + \lambda_f W f_t + \lambda_s W z_t.
\end{equation}

Since $s_t$, $f_t$, and $z_t$ are bounded and $W$ is finite-dimensional, $I_t$
remains bounded. This prevents membrane explosion and ensures numerical
stability during long unrolls.

\section{Comparison to Adaptive Membrane Models}

Adaptive membrane models modify neuron dynamics:
\begin{equation}
v_t = \alpha_t v_{t-1} + I_t,
\end{equation}
where $\alpha_t$ adapts slowly.

In contrast, CPSNNs:
\begin{itemize}
\item Adapt at the synaptic level, not neuron level
\item Allow per-input temporal control
\item Preserve event-driven sparsity
\item Avoid global coupling between inputs
\end{itemize}

Formally, CPSNNs implement a mixture of temporal bases, while adaptive
membrane models implement a single evolving timescale.

\section{Ablation Justification}

We justify the ablation experiments reported in the main paper:

\begin{itemize}
\item \textbf{No warp:} Setting $\omega_t \equiv 1$ reduces CPSNN to a fixed slow
trace, eliminating adaptive time warping.
\item \textbf{No slow trace:} Removing $z_t$ collapses CPSNN to a short-memory
model.
\item \textbf{No fast trace:} Removes short-term sensitivity, degrading
responsiveness.
\end{itemize}

Each component is therefore necessary for balancing responsiveness and long-term
memory.

The ablation results are best understood by analyzing how each removed
component alters the underlying temporal dynamics and credit assignment
properties of the model.

Fixing the warp factor reduces the slow trace update to
\begin{equation}
z_t = \alpha_s z_{t-1} + s_t,
\end{equation}
which recovers a standard exponentially decaying synaptic trace. In this
setting, the effective memory kernel becomes stationary and depends only on
elapsed physical time. As shown in Section~A, this imposes a hard upper bound
on the effective memory horizon and reintroduces exponential gradient decay
for long delays. Empirically, this variant behaves similarly to fixed-decay
SNNs: it can either preserve information by choosing $\alpha_s$ close to one,
at the cost of noise accumulation, or remain responsive by choosing smaller
$\alpha_s$, at the cost of forgetting. The inability to adapt this trade-off
online explains the observed performance degradation.

Removing the slow trace eliminates the long-timescale memory pathway
altogether, leaving only the fast trace
\begin{equation}
f_t = \alpha_f f_{t-1} + s_t,
\end{equation}
whose effective support is limited to a short temporal window. In this
configuration, CPSNN collapses to a purely short-memory model, incapable of
bridging long temporal gaps regardless of the learned warp dynamics. From a
functional standpoint, this ablation removes the model’s ability to represent
temporal dependencies whose scale exceeds the fast decay constant, making
long-range credit assignment fundamentally impossible.

Removing the fast trace preserves long-term memory through $z_t$ but
significantly degrades short-term sensitivity. In this case, all synaptic
integration occurs through the slow pathway, whose decay—even when warped—is
necessarily smoother and less responsive to rapid input fluctuations. This
reduces the model’s ability to react to closely spaced events, increasing
latency and impairing discrimination of short-term temporal patterns.
Empirically, this manifests as slower convergence and reduced accuracy in
tasks requiring both immediate responsiveness and delayed integration.

The full CPSNN architecture can be viewed as a structured decomposition of
temporal processing into complementary regimes. The fast trace provides a
high-bandwidth channel for recent activity, while the slow trace—modulated by
the warp factor—acts as a selectively persistent memory. The synaptic current
\begin{equation}
I_t = W s_t + \lambda_f W f_t + \lambda_s W z_t
\end{equation}
combines these pathways linearly, allowing the neuron to integrate information
across multiple timescales without explicit gating or control flow.

Only the full CPSNN configuration simultaneously satisfies three essential
properties: (i) responsiveness to recent inputs, (ii) preservation of
task-relevant information across long gaps, and (iii) stability under noisy or
distractor-rich input streams. Removing any component breaks this balance by
either collapsing the effective memory horizon, introducing excessive inertia,
or eliminating adaptive temporal modulation. The ablation results therefore
validate not only the necessity of each component, but also the architectural
principle that adaptive temporal learning in spiking systems requires
both multi-timescale state and learned time modulation.

\section{Computational Complexity}

We analyze the computational and memory complexity of CPSNNs and contrast them
with standard spiking networks and attention-based temporal models. Throughout
this section, let $N$ denote the number of neurons, $C$ the number of
presynaptic input channels per neuron, and $T$ the length of the input
sequence.

\subsection{Time Complexity}

A standard spiking neural network (SNN) with leaky integrate-and-fire neurons
performs, at each timestep, a synaptic accumulation followed by a membrane
update and spike generation. For a fully connected layer, this requires
$\mathcal{O}(N C)$ operations per timestep, yielding an overall complexity of
\begin{equation}
\mathcal{O}(T \cdot N \cdot C).
\end{equation}

CPSNNs preserve this asymptotic scaling. Each ChronoPlastic synapse augments a
standard synapse with two additional trace updates,
\begin{equation}
f_t = \alpha_f f_{t-1} + s_t, \qquad
z_t = \alpha_s^{\omega_t} z_{t-1} + s_t,
\end{equation}
and evaluates a lightweight control function
\begin{equation}
\omega_t = \sigma(g([s_t, z_{t-1}])).
\end{equation}
All of these operations are local and linear in the number of input channels.
Consequently, the per-timestep cost remains $\mathcal{O}(N C)$, and the total
runtime complexity of CPSNNs is
\begin{equation}
\mathcal{O}(T \cdot N \cdot C),
\end{equation}
matching that of standard SNNs up to a small constant factor.

In contrast, attention-based temporal models explicitly compute interactions
between all pairs of timesteps. Even for a single layer, this incurs a cost of
$\mathcal{O}(T^2)$ per neuron or feature dimension, resulting in an overall
complexity of
\begin{equation}
\mathcal{O}(T^2 \cdot N),
\end{equation}
which becomes prohibitive for long sequences. CPSNNs avoid this quadratic
scaling entirely by embedding temporal adaptivity directly into synaptic
dynamics rather than performing explicit temporal comparisons.

The control network responsible for generating $\omega_t$ introduces only
$\mathcal{O}(C)$ additional operations per synapse per timestep. Because this
cost is independent of $T$ and does not introduce global communication, it does
not alter the asymptotic complexity of the model.

\subsection{Space Complexity}

From a memory perspective, standard SNNs store a constant amount of state per
synapse and neuron, resulting in $\mathcal{O}(N C)$ memory usage independent of
sequence length.

CPSNNs introduce two additional scalar state variables per synapse—the fast
trace $f_t$ and slow trace $z_t$—along with a small number of parameters for
the control network. The total memory required by CPSNNs therefore scales as
\begin{equation}
\mathcal{O}(N \cdot C),
\end{equation}
which is linear in network size and independent of the temporal horizon $T$.
Importantly, CPSNNs do not require storing past activations, attention keys, or
external memory buffers whose size grows with sequence length.

This property enables CPSNNs to process arbitrarily long temporal streams
without increasing memory consumption, making them particularly well suited
for streaming and online inference settings. In contrast, attention-based and
external-memory architectures typically require memory that scales at least
linearly with $T$, and often quadratically when intermediate activations are
stored for backpropagation.

The combination of linear-time execution and constant-memory temporal state
makes CPSNNs uniquely scalable among models capable of long-horizon temporal
learning. Temporal capacity in CPSNNs is not achieved by increasing network
depth, widening layers, or storing longer histories, but by dynamically
modulating synaptic decay. As a result, CPSNNs decouple effective memory horizon
from both architectural size and computational cost.

This efficiency is especially important for neuromorphic hardware and
event-driven systems, where local state updates, sparse communication, and
constant-time operations are critical constraints. CPSNNs satisfy these
requirements while still providing a mechanism for adaptive, task-driven
temporal memory.

\section{Implementation Details}

This section describes the practical implementation choices used in all
experiments, including hyperparameter settings, architectural design of the
control network, and training procedures. 

\subsection{Hyperparameters}

Unless stated otherwise, all CPSNN models use the following default
hyperparameters. These values were selected based on preliminary validation
experiments and kept fixed across all reported results.

\paragraph{Synaptic decay constants.}
The fast and slow trace decay parameters are set to
\begin{equation}
\alpha_f = 0.9, \qquad \alpha_s = 0.995.
\end{equation}
This choice ensures a clear separation of timescales. The fast trace $f_t$
decays rapidly, capturing short-term temporal correlations over a small number
of timesteps, while the slow trace $z_t$ provides a long integration window
capable of spanning hundreds to thousands of timesteps when modulated by the
adaptive warp factor. Importantly, $\alpha_s$ is chosen close to $1$ to allow
substantial memory extension when $\omega_t$ is small, without risking
instability.

\paragraph{Trace mixing coefficients.}
The relative contributions of fast and slow traces to the synaptic current are
controlled by scalar coefficients
\begin{equation}
\lambda_f = \lambda_s = 0.5.
\end{equation}
This symmetric weighting provides a balanced contribution from short-term and
long-term memory pathways. In practice, we find that CPSNN performance is robust
to moderate variation in these coefficients, indicating that the adaptive warp
mechanism, rather than precise tuning of $\lambda_f$ and $\lambda_s$, is the
dominant factor governing temporal behavior.

\paragraph{Control network architecture.}
The adaptive warp factor $\omega_t$ is generated by a lightweight control
network of the form
\begin{equation}
\omega_t = \sigma\big(W_c [s_t, z_{t-1}] + b_c\big),
\end{equation}
where $W_c$ and $b_c$ denote trainable parameters and $\sigma(\cdot)$ is the
sigmoid function. This control network consists of a single linear layer
followed by a sigmoid nonlinearity, ensuring minimal computational overhead and
bounded outputs in $(0,1)$.

We deliberately restrict the control network to a shallow architecture. This
design choice prevents overfitting, preserves locality, and ensures that
temporal modulation remains interpretable. Despite its simplicity, this module
is sufficient to learn rich, input-conditioned temporal warping strategies, as
demonstrated empirically.

\paragraph{Initialization.}
All synaptic weights are initialized using standard variance-preserving
initialization. Trace variables $f_0$ and $z_0$ are initialized to zero at the
start of each sequence. Control network parameters are initialized such that
$\omega_t \approx 1$ at initialization, ensuring that CPSNNs initially behave
similarly to fixed-decay SNNs and gradually learn adaptive temporal modulation
during training.

\subsection{Training}

All models are trained end-to-end using backpropagation through time (BPTT) with
surrogate gradients. Training procedures are identical across CPSNNs and all
baselines to ensure fair comparison.

\paragraph{Optimizer and learning rate.}
We use the Adam optimizer with learning rate
\begin{equation}
\eta = 10^{-2},
\end{equation}
and default momentum parameters $(\beta_1 = 0.9, \beta_2 = 0.999)$. Adam is
particularly well suited to CPSNNs because gradients associated with long-term
memory pathways may be sparse or delayed, and adaptive learning rates help
stabilize optimization in this regime.

\paragraph{Gradient clipping.}
To prevent exploding gradients during long temporal unrolls, gradients are
clipped to a maximum $\ell_2$ norm of $1.0$. This clipping is applied uniformly
across all parameters, including synaptic weights and control network
parameters. We find that gradient clipping is essential for stable training
when sequence lengths and temporal gaps are large.

\paragraph{Surrogate gradients.}
Spike generation functions are non-differentiable. During the backward pass, we
replace the derivative of the Heaviside step function with a smooth surrogate,
typically a bounded piecewise-linear or sigmoid-shaped function. This approach
enables gradients to flow through spike events while preserving the discrete
spiking behavior in the forward pass.

\paragraph{Sequence unrolling and batching.}
Training is performed by unrolling network dynamics over the full temporal
horizon $T$ and computing gradients via BPTT. Mini-batches are constructed by
grouping sequences of identical length. For tasks with variable temporal gaps,
the maximum unroll length is fixed per experiment to ensure consistent memory
usage across models.

\paragraph{Reproducibility.}
All experiments use identical random seeds across models, including weight
initialization, input generation, and mini-batch ordering. This ensures that
performance differences arise from architectural properties rather than random
variation. Reported results are averaged over multiple independent runs when
appropriate.

\paragraph{Training stability.}
We observe that CPSNNs exhibit more stable optimization dynamics than fixed
decay SNNs, particularly for long temporal gaps. Adaptive temporal modulation
reduces gradient noise by aggressively forgetting irrelevant activity while
preserving gradients associated with informative events. This behavior leads to
faster convergence and lower variance across training runs.

\paragraph{Computational considerations.}
Despite the additional trace updates and control network evaluation, CPSNNs
incur only a small constant-factor overhead compared to standard SNNs. All
operations remain local and event-driven, and no global synchronization or
sequence-level operations are introduced. Consequently, CPSNNs scale well to
long sequences and large networks and are compatible with neuromorphic and
streaming inference settings.

Overall, these implementation choices ensure that CPSNNs are both easy to train
and computationally efficient, while fully exploiting adaptive temporal
modulation to achieve superior long-horizon learning.

\section{Biological Interpretation}

ChronoPlastic synapses are inspired by a growing body of evidence that biological
synaptic computation is inherently multi-timescale and dynamically regulated by
activity-dependent mechanisms. Rather than operating with a single fixed decay
constant, biological synapses express a spectrum of temporal dynamics through
short-term plasticity, long-term plasticity, and neuromodulatory control. CPSNNs
abstract and unify these mechanisms into a single computational framework that
remains both biologically plausible and algorithmically efficient.

\subsection{Multi-Timescale Synaptic Dynamics}

Experimental neuroscience has long established that synapses exhibit multiple
coexisting temporal processes, including facilitation, depression, and
metaplasticity, each operating on distinct timescales. Short-term synaptic
facilitation and depression modulate synaptic efficacy over milliseconds to
seconds, while longer-term processes such as synaptic consolidation and decay
can persist for minutes to hours or longer.

In CPSNNs, the fast trace $f_t$ corresponds to short-term synaptic dynamics that
capture recent presynaptic activity. This trace resembles transient facilitation
or residual calcium effects, where recent spikes temporarily increase synaptic
efficacy before rapidly decaying. The slow trace $z_t$, in contrast, serves as a
longer-lasting memory variable, analogous to slower biochemical processes such as
protein phosphorylation states or structural synaptic changes that decay over
extended periods.

Crucially, CPSNNs do not hard-code the interaction between these timescales.
Instead, they allow the relative influence of fast and slow dynamics to be
learned and dynamically modulated, reflecting the flexible and context-dependent
nature of biological synaptic integration.

\subsection{Warp Factor as Neuromodulatory Control}

The adaptive warp factor $\omega_t$ admits a biologically meaningful
interpretation as a form of activity-dependent neuromodulation. In biological
systems, neuromodulators such as dopamine, acetylcholine, and norepinephrine
regulate synaptic persistence, plasticity thresholds, and learning rates rather
than directly encoding information themselves. These modulators act locally in
time and space, altering how long synaptic traces remain influential.

In CPSNNs, $\omega_t$ plays an analogous role by regulating the effective decay
rate of the slow synaptic trace. When $\omega_t$ is small, synaptic decay is
slowed, allowing information to persist across long temporal gaps. When
$\omega_t$ is large, decay accelerates, rapidly suppressing irrelevant or noisy
activity. This mirrors biological findings that neuromodulatory signals can gate
memory persistence based on behavioral relevance, reward prediction error, or
contextual salience.

Importantly, $\omega_t$ is computed from local synaptic state and presynaptic
activity, rather than from a global signal. This aligns with evidence that
neuromodulatory effects, while sometimes broadcast, are often filtered and
interpreted locally at the synapse through receptor-specific pathways and
intracellular signaling cascades.

\subsection{Metaplasticity and Adaptive Memory Allocation}

Biological synapses do not merely change their strength; they also change how
they change. This phenomenon, known as metaplasticity, refers to the regulation
of plasticity rules themselves based on prior activity. CPSNNs capture an
analogous principle through adaptive time warping.

Rather than adjusting synaptic weights directly, the warp factor modifies the
temporal persistence of synaptic traces. This effectively alters the window over
which past activity influences current computation and learning. From a
biological perspective, this corresponds to adjusting the temporal sensitivity
of a synapse based on recent history, a behavior observed in cortical and
hippocampal circuits.

Because $\omega_t$ is learned end-to-end and remains differentiable, CPSNNs
implement a continuous form of metaplasticity that adapts in response to task
demands. This contrasts with classical models of metaplasticity that rely on
discrete state transitions or hand-crafted rules.

\subsection{Relation to Working Memory and Synaptic Theories}

Several synaptic theories of working memory propose that transient synaptic
states, rather than persistent neuronal firing, store information across delays.
In these models, information is maintained silently in synaptic variables and
reactivated when needed. CPSNNs naturally align with this view.

The slow trace $z_t$ functions as a silent memory store that can preserve
information without sustained spiking activity. Adaptive time warping further
enhances this mechanism by extending synaptic persistence selectively, enabling
long delays to be bridged without continuous neural activation. This offers a
computational explanation for how biological systems maintain working memory
with minimal metabolic cost.

Unlike models that require explicit bistable synaptic states or discrete memory
buffers, CPSNNs maintain continuous, bounded synaptic variables that evolve
smoothly over time, aligning with known biochemical constraints.

\subsection{Event-Driven Computation and Energy Efficiency}

A defining characteristic of biological neural systems is their reliance on
event-driven computation. Neurons and synapses remain largely quiescent in the
absence of meaningful input, conserving energy. CPSNNs preserve this principle
by operating entirely on spike events and local state updates.

Adaptive temporal modulation further enhances energy efficiency by preventing
unnecessary persistence of synaptic activity. When inputs are uninformative,
$\omega_t$ increases, accelerating decay and reducing downstream computation.
When salient events occur, decay slows only where needed. This selective memory
allocation mirrors biological strategies for balancing metabolic cost and
computational performance.

\subsection{Abstraction Level and Limitations}

While CPSNNs are biologically inspired, they are not intended as a detailed
biophysical model of synaptic chemistry or neuromodulation. The warp factor
abstracts many interacting biological processes into a single scalar control
signal. This abstraction sacrifices mechanistic detail in favor of computational
clarity and scalability.

Nevertheless, CPSNNs occupy an intermediate level of biological realism: more
expressive than fixed-timescale SNNs, yet far simpler than full conductance-based
or biochemical synapse models. This balance allows CPSNNs to capture key
principles of biological temporal computation—multi-timescale dynamics,
context-dependent memory persistence, and local modulation—while remaining
tractable for large-scale learning and neuromorphic deployment.

Overall, the ChronoPlastic framework suggests that adaptive temporal control at
the synaptic level may be a fundamental mechanism underlying both biological
temporal cognition and efficient artificial spiking intelligence.

\section{Limitations and Future Directions}

While ChronoPlastic Spiking Neural Networks introduce a powerful and scalable
mechanism for adaptive temporal learning, several limitations remain. These
limitations highlight natural extensions of the framework rather than
fundamental obstacles, and point toward a broader research program around
adaptive synaptic time modulation.

\subsection{Single Slow Trace Representation}

In the current formulation, each ChronoPlastic synapse maintains a single slow
trace $z_t$ whose decay is adaptively modulated via the warp factor $\omega_t$.
While this is sufficient to demonstrate substantial gains on long-gap temporal
tasks, biological and cognitive systems are known to operate across multiple
distinct temporal regimes simultaneously.

A single slow trace effectively learns a dominant long-term timescale, but
cannot explicitly represent multiple concurrent memory horizons (e.g.,
seconds, minutes, and hours) within the same synapse. This may limit
representational flexibility on tasks requiring hierarchical temporal
reasoning, such as nested event structure or multi-stage decision processes.

A natural extension is to introduce multi-level ChronoPlastic traces,
\begin{equation}
z_t^{(k)} = \alpha_k^{\omega_t^{(k)}} z_{t-1}^{(k)} + s_t,
\end{equation}
where each trace operates at a different base timescale $\alpha_k$ and is
modulated by its own warp factor. Such a hierarchy would allow CPSNNs to learn a
rich temporal basis spanning short-, medium-, and long-term dependencies, while
preserving locality and linear-time complexity.

\subsection{Local Versus Shared Warp Controllers}

In this work, warp factors are computed independently at each synapse using a
local control network conditioned on presynaptic activity and synaptic state.
While this maximizes flexibility, it also increases parameter count and may
introduce redundancy in large networks.

Biological neuromodulatory systems often exhibit partial spatial sharing, where
groups of synapses receive correlated modulatory signals. Inspired by this,
future work could explore spatially shared or factorized warp controllers,
in which a single control signal modulates multiple synapses or neurons within a
layer or functional group.

Such shared controllers could reduce parameter overhead, improve statistical
efficiency, and encourage coordinated temporal behavior across populations of
neurons, while still avoiding global attention or centralized memory mechanisms.
Understanding the trade-off between local autonomy and shared temporal control
is an important open question.

\subsection{Interaction with Learning Rules and Plasticity}

CPSNNs currently focus on adaptive temporal modulation of synaptic traces, while
synaptic weights are updated using standard gradient-based optimization with
surrogate gradients. This separation simplifies analysis but does not exploit
the full potential of biologically inspired plasticity mechanisms.

Future extensions could integrate ChronoPlastic dynamics with local plasticity
rules such as spike-timing-dependent plasticity (STDP), three-factor learning
rules, or reward-modulated updates. In such settings, the warp factor could
control not only synaptic persistence but also learning rates or eligibility
trace decay, providing a unified framework for temporal credit assignment and
plasticity modulation.

Exploring these interactions may yield models that are both more biologically
faithful and more robust under sparse or delayed supervision.

\subsection{Hardware-Aware Constraints and Quantization}

Although CPSNNs are designed to be compatible with neuromorphic hardware, the
current implementation assumes floating-point arithmetic and unconstrained
precision. Practical deployment on low-power neuromorphic processors requires
careful consideration of quantization, limited precision, and fixed-point
representations.

The exponential decay term $\alpha_s^{\omega_t}$, while mathematically smooth,
may be approximated or discretized for efficient hardware execution. Future work
should investigate hardware-aware approximations, lookup-table implementations,
or piecewise-linear surrogates that preserve adaptive temporal behavior while
meeting strict power and area constraints.

Such co-design between ChronoPlastic algorithms and neuromorphic hardware could
unlock significant efficiency gains and enable real-time deployment in embedded
settings.

\subsection{Task Scope and Generalization}

Our empirical evaluation focuses on controlled synthetic benchmarks that isolate
long-range temporal credit assignment. While this choice enables clear causal
analysis, it does not exhaust the range of temporal challenges encountered in
real-world applications.

Extending CPSNNs to complex sensory streams, multi-modal event data, and
closed-loop control tasks remains an important direction for future research.
In particular, evaluating CPSNNs on neuromorphic vision datasets, continuous
control benchmarks, and reinforcement learning settings will help clarify the
generality of adaptive time warping as a core computational principle.

\subsection{Theoretical Characterization}

Finally, while we provide initial analysis of gradient flow, boundedness, and
expressive advantages over fixed-timescale SNNs, a complete theoretical
characterization of CPSNNs remains open. Formal results on approximation power,
memory capacity, and optimal warp policies would further strengthen the
foundation of the framework.

In summary, CPSNNs represent an initial instantiation of a broader idea:
temporal learning through adaptive synaptic time modulation. Addressing
the limitations outlined above offers a clear and tractable path toward richer,
more expressive, and more biologically grounded spiking architectures.

\end{document}